\documentclass[sigconf]{acmart}
\AtBeginDocument{%
  }

\copyrightyear{2026}
\acmYear{2026}
\setcopyright{cc}
\setcctype{by}
\acmConference[KDD 2026] {Proceedings of the 32nd ACM SIGKDD Conference on Knowledge Discovery and Data Mining V.2}{August 9--13, 2026}{Jeju Island, Republic of Korea.}
\acmBooktitle{Proceedings of the 32nd ACM SIGKDD Conference on Knowledge Discovery and Data Mining V.2 (KDD 2026), August 9--13, 2026, Jeju Island, Republic of Korea}
\acmISBN{979-8-4007-2259-2/2026/08}
\acmDOI{10.1145/3770855.3817820}

\usepackage{enumitem}
\usepackage{hyperref}
\usepackage{amsthm}     
\usepackage{cleveref}
\usepackage{booktabs}
\usepackage[table]{xcolor}
\usepackage{graphicx}
\usepackage{subcaption}

\usepackage{algorithm} 
\usepackage{algorithmic} 
\usepackage{multirow} 

\usepackage{pifont} 
\newcommand{\cmark}{\ding{51}}  
\newcommand{\xmark}{\ding{55}}  

\newtheoremstyle{colon}{}{}{}{}{\bfseries}{: }{ }{}
\theoremstyle{colon}

\newtheorem{definition}{Definition}[section]
\newtheorem{theorem}{Theorem}[section]   
\newtheorem*{problem}{Problem}           

\crefname{section}{Section}{Sections}
\crefname{figure}{Figure}{Figures}
\crefname{table}{Table}{Tables}
\crefname{theorem}{Theorem}{Theorems}
\crefname{definition}{Definition}{Definitions}

\begin{document}

\title{HyperPatch: Sequential Knowledge Editing Under $n$-ary Structural Drift}


\author{Yu-Kai Chan}
\orcid{0009-0005-0468-2525}
\affiliation{%
  \institution{National Yang Ming Chiao Tung University}
  \city{Hsinchu}
  \country{Taiwan}
}
\email{ctw33888@gmail.com}

\author{Wen-Sheng Lien}
\orcid{0009-0004-1107-5590}
\affiliation{%
  \institution{National Yang Ming Chiao Tung University}
  \city{Hsinchu}
  \country{Taiwan}
}
\email{vincentlien.ii13@nycu.edu.tw}

\author{Dong-Ting Yao}
\orcid{0009-0004-8106-4715}
\affiliation{%
  \institution{National Yang Ming Chiao Tung University}
  \city{Hsinchu}
  \country{Taiwan}
}
\email{ryanyao.cs14@nycu.edu.tw}

\author{Bo-Kai Ruan}
\orcid{0000-0002-9847-3628}
\affiliation{%
  \institution{National Yang Ming Chiao Tung University}
  \city{Hsinchu}
  \country{Taiwan}
}
\email{bkruan.ee11@nycu.edu.tw}

\author{Kwan-Yeung Lin}
\orcid{0009-0009-1956-6171}
\affiliation{%
  \institution{National Yang Ming Chiao Tung University}
  \city{Hsinchu}
  \country{Taiwan}
}
\email{keithlin.ee13@nycu.edu.tw}

\author{Hong-Han Shuai}
\orcid{0000-0003-2216-077X}
\affiliation{%
  \institution{National Yang Ming Chiao Tung University}
  \city{Hsinchu}
  \country{Taiwan}
}
\email{hhshuai@nycu.edu.tw}

\author{Meng-Fen Chiang}
\orcid{0009-0008-8385-0380}
\affiliation{%
  \institution{National Yang Ming Chiao Tung University}
  \city{Hsinchu}
  \country{Taiwan}
}
\email{meng.chiang@nycu.edu.tw}
\authornote{Corresponding author.}

\renewcommand{\shortauthors}{Chan et al.}

\newcommand{\method}{HyperPatch}
\newcommand{\drift}{Structural Drift}
\newcommand{\sktf}{SKTF}
\newcommand{\nary}{$n$-ary}
\begin{abstract}
Large Language Models (LLMs) rely on Knowledge Editing (KE) to maintain temporal validity, yet real-world knowledge is inherently \nary{}. We demonstrate that in non-stationary environments, sequential updates to complex relations induce \textit{\drift{}}, a phenomenon where the binary reification of \nary{} events into triples fractures relational atomicity. This precipitates \textit{Structure-Conditioned Knowledge Transfer Failure} (\sktf{}), a systematic mis-grounding of the retriever frequently misdiagnosed as parametric hallucination. To tackle this, we propose \textbf{\method{}}, a parameter-preserving framework that reformulates sequential KE as a stability problem over hypergraph manifolds. \method{} preserves event integrity through three phases: (i) \textbf{Structural Prior Initialization}, establishing a topology-aware embedding space via contrastive learning on a Hypergraph Neural Network (HGNN) to capture high-order correlations; (ii) \textbf{Sequential Topology Editing}, utilizing a dual-stage mechanism that employs SimHash-based Topological Alignment for rapid conflict resolution and Topological LoRA Adaptation to track drift without backbone retraining; and (iii) \textbf{Structure-Conditioned Reasoning}, which integrates globally consistent evidence from fused linguistic and structural manifolds. On the MQuAKE-CF and MQuAKE-T benchmarks, \method{} achieves relative gains in Hop-wise Accuracy (H-Acc) of 96.24\% and 21.06\% over the strongest baseline, respectively.
Further ablations demonstrate superior reliability under continuous $n$-ary update streams, whereas the standard KG-based variant suffers H-Acc collapses of up to 88.3\% due to structural misalignment. 
\noindent Our code is publicly available at 
\url{https://github.com/Kevin20010912/HyperPatch.git}.
\end{abstract}




\maketitle
\newcommand\kddavailabilityurl{https://doi.org/10.5281/zenodo.20372801}
\ifdefempty{\kddavailabilityurl}{}{
\begingroup\small\noindent\raggedright\textbf{Resource Availability:}\\
The source code of this paper has been made publicly available at \url{\kddavailabilityurl}.
\endgroup
}

\section{Introduction}\label{sec:intro}
Large Language Models (LLMs) are increasingly used as the generator in multi-hop Question Answering (QA) pipelines \cite{zhong2023mquake}, where answers are produced by conditioning on external knowledge sources (\textit{e.g.}, retrieved text or graph facts). In such systems, end-to-end reliability depends not only on the parametric reasoning of LLMs, but also on whether the external knowledge being accessed remains current and consistently reachable through the system’s representation and retrieval interfaces. When facts evolve, a practical maintenance primitive is \textit{Knowledge Editing} (KE): performing localized updates that inject new facts while avoiding global retraining and minimizing collateral changes \cite{meng2022locating,meng2023mass}. Ideally, KE facilitates a ``clean contract'': newly injected facts should be immediately actionable for downstream reasoning while maintaining global consistency.


\begin{figure}[t]
\centering
    
    \begin{subfigure}[b]{1\columnwidth}
        \centering
        \includegraphics[width=\textwidth]{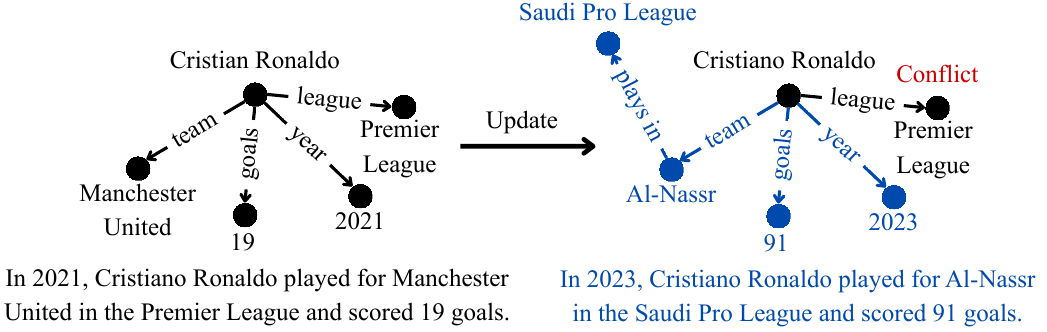}
        \caption{}
        \label{fig:sub1}
    \end{subfigure}
    
    \begin{subfigure}[b]{0.5\columnwidth}
        \centering
        \includegraphics[width=\textwidth]{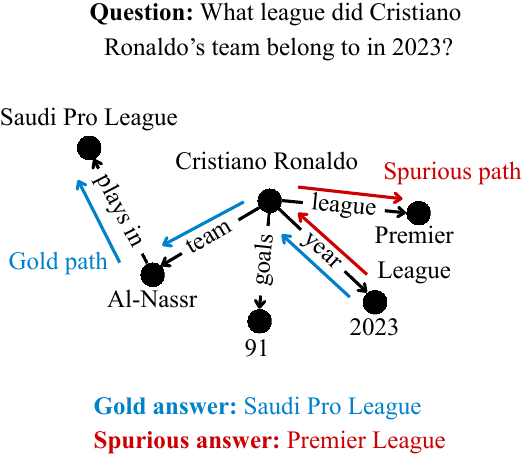}
        \caption{}
        \label{fig:sub2}
    \end{subfigure}
    \begin{subfigure}[b]{0.47\columnwidth}
        \centering
        \includegraphics[width=\textwidth]{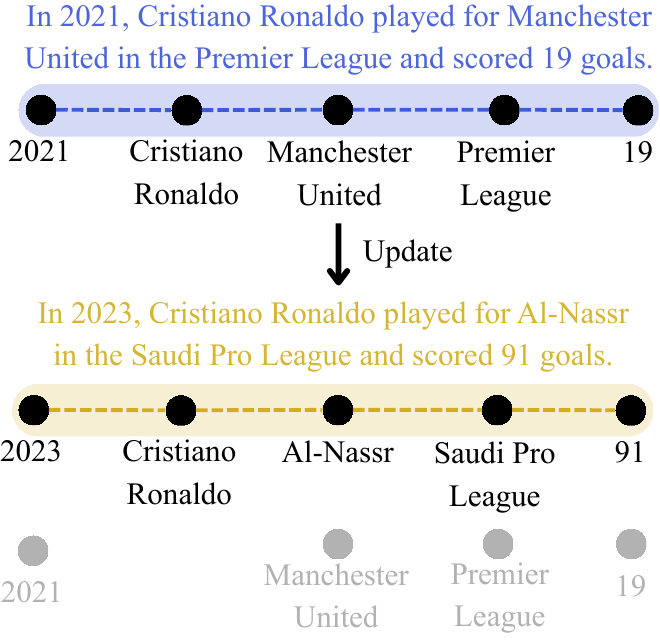}
        \caption{}
        \label{fig:sub3}
    \end{subfigure}     
\caption{SKTF under $n$-ary structural drift. (a) Factorization to binary edges breaks event coupling and retrieves fragments. (b) Binary walks create locally similar but globally invalid compositions. (c) \method{} retrieves atomic hyperedges, preserving event consistency across updates.}
\label{fig:motivation}
\end{figure}

Most prior work probes this contract in \emph{static} regimes, implicitly assuming that the knowledge representation and the retriever’s embedding manifold remain aligned over time \cite{meng2022locating,meng2023mass}. Real deployments, however, experience sequential updates and schema evolution that introduce \textit{structural drift} \cite{gama2014survey}. We highlight a failure mode under such drift: \textbf{Structure-Conditioned Knowledge Transfer Failure (SKTF)}. SKTF arises when the topology mediating knowledge access (\textit{e.g.}, how complex relations are factorized and connected) diverges from the retriever’s learned embedding distribution. As a result, the system may retrieve stale or structurally incompatible evidence despite the underlying knowledge being correctly updated, mis-grounding the generator and producing confident errors that are often misattributed to hallucination or parametric conflicts \cite{chen2022rich}. We argue that, under $n$-ary structural shifts, these errors primarily reflect \emph{retrieval-manifold misalignment} rather than an inherent deficiency in the reasoning capacity.


Specifically, SKTF stems from a mismatch between how facts arise as \emph{events} and how they are represented for retrieval. Real-world knowledge is often \textit{$n$-ary}: an event jointly binds multiple entities, roles, attributes, and contextual qualifiers \cite{wen2016learning}. In contrast, standard Knowledge Graph (KG) pipelines typically \emph{project} or \emph{reify} such events into sets of binary triples \cite{rosso2020beyond} to support link-based storage and embedding-based retrieval. While this approximation is often acceptable when the schema and data distribution are stable, it becomes brittle under sequential updates and evolving entity/query distributions. Figure~\ref{fig:motivation} illustrates two drift mechanisms: \textbf{(a) Event Factorization Drift}, where decomposing an $n$-ary event into loosely-coupled edges weakens semantic coupling. Under distribution shift (\textit{e.g.}, a 2023 update to a player's profile), the retriever’s neighborhood is reshaped. Fragments such as the \textit{Ronaldo} entity and the \textit{Premier League} relation may remain proximal in embedding space despite belonging to different underlying events. This yields evidence that is locally relevant yet \emph{event-incomplete} \cite{thakur2021beir}, as the retriever can no longer enforce the original 
\nary{} constraints. \textbf{(b) Spurious Multi-hop Composition}, where binary projections admit accidental multi-hop paths that appear plausible hop-by-hop but fail to form a valid global chain when composed \cite{min2019compositional}. In both cases, retrieval returns evidence that is locally plausible but globally inconsistent, so the generator is mis-grounded not because the required knowledge is absent, but because the system assembles an invalid event structure from drifted binary proxies.

Despite the urgency of SKTF, existing literature overlooks the structural dimension of editing. \textit{Parametric Editing} \cite{meng2022locating, meng2023mass} effectively modifies single-hop facts but struggles with distributed evidence \cite{zhong2023mquake} and suffers from catastrophic forgetting under sequential streams \cite{kirkpatrick2017overcoming}. Conversely, \textit{Retrieval-Augmented Editing} \cite{mitchell2022memory, han2023memory} relies on frozen, dense retrievers that are agnostic to evolving graph topology. These methods assume embedding spaces remain aligned with knowledge structures, which is an assumption that fails as $n$-ary relations drift. Currently, no framework explicitly models the \textit{stability of event compatibility} during sequential editing. This motivates a shift toward \textbf{Hypergraphs}. Unlike standard KGs, a hypergraph treats each $n$-ary event as a single hyperedge, preserving the atomic unit that couples participants and constraints \cite{feng2019hypergraph, haghighi2024tropical}. Hypergraph-based Retrieval-Augmented Generation (RAG)\cite{luo2025hypergraphrag} can thus retrieve evidence as ``event-consistent units'' rather than arbitrary triple chains, mitigating spurious composition \cite{edge2024global}. Crucially for editing, hypergraphs allow drift to be measured at the level of event compatibility. As shown in Fig. \ref{fig:motivation}(c), while standard KGs fragment context into drifting triples, hyperedge representations maintain structural integrity. This ensures that updates to a single participant do not decouple associated context, providing a robust substrate for scalable, sequential knowledge editing.

In this paper, we reformulate sequential knowledge editing as a topological stability problem under $n$-ary structural drift. We posit that the primary object to be preserved is the \textit{event compatibility manifold}, which is the latent structure governing how $n$-ary relations are retrieved and composed. To this end, we introduce \textbf{\method{}}, a parameter-preserving framework that stabilizes hyperedge representations through three core innovations:
(i) \textbf{Structural Prior Initialization:} To mitigate SKTF, we synchronize the retriever's embedding manifold with the evolving hypergraph topology via an HGNN and topological LoRA adaptation. By grounding updates in $n$-ary structural constraints, this replay-free mechanism circumvents the catastrophic forgetting inherent in sequential editing \cite{hartvigsen2024aging}, maintaining a robust structural anchor as the knowledge base evolves.
(ii) \textbf{Sequential Topology Editing:} To tackle $n$-ary updates at scale, we implement a retrieval layer based on hyperedge-centric hashing. This facilitates \textbf{$O(1)$} structural mapping for rapid hyperedge replacement or expansion, effectively decoupling update latency from the cumulative edit volume.
(iii) \textbf{Structure-Conditioned Reasoning:} 
For drift-resilient inference, \method{} dynamically integrates retrieval manifolds with updated $n$-ary topologies. This facilitates precise entity grounding and consistent traversal of reasoning paths, maintaining multi-hop QA performance even under stochastic structural fluctuations.

Our contributions are as follows:
\begin{itemize}[leftmargin=*]
    \item We identify SKTF as a dominant error source in dynamic multi-hop QA and provide a theoretical bound linking retrieval drift to reasoning risk in $n$-ary settings.
    \item We propose \textbf{\method{}}, the first framework to integrate hypergraph structural priors into sequential KE. By preserving event integrity, \method{} robustly mitigates the $n$-ary factorization drift that fractures standard triple-based editors.
    \item Evaluations on multi-hop benchmarks demonstrate that \method{} achieves superior reliability (96.24\% and 21.06\% H-Acc) and maintains robustness to distributional shifts, mitigating the catastrophic 42.9\%--88.3\% H-Acc degradation inherent in KG-based variant. Simultaneously, \method{} delivers a $25.9\times$ retrieval speedup over competitive knowledge editing frameworks.
\end{itemize}

\begin{figure*}[!ht]
    \centering
    \includegraphics[width=0.85\linewidth]{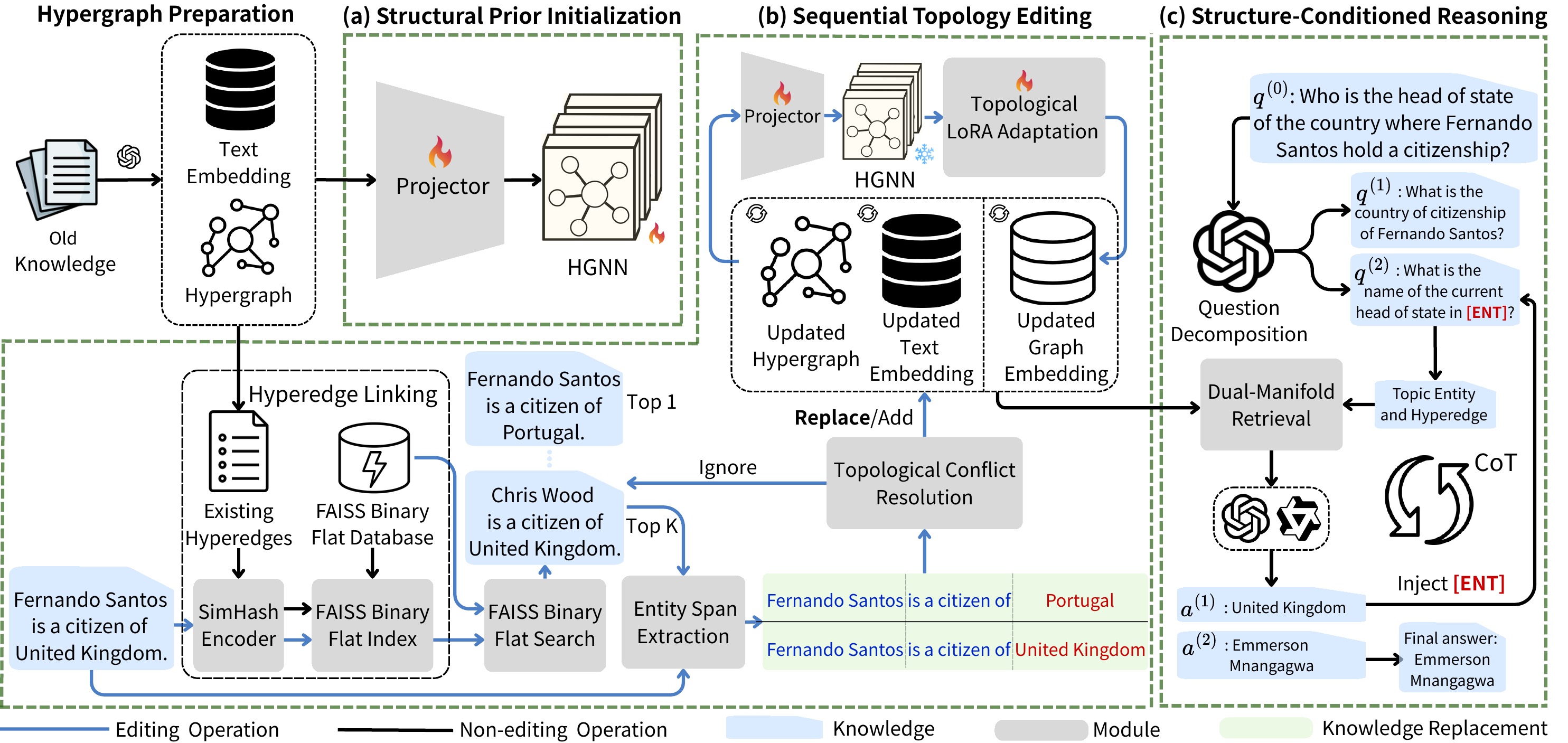}
    \caption{The overall framework of HyperPatch.}
    \label{fig:framework}
\end{figure*}

\section{Preliminaries}\label{sec:pre}
We formally define the representation of $n$-ary knowledge (\cref{sec:rkg}), the sequential editing setting (\cref{sec:ke}), and the structural drift that challenges contemporary systems (\cref{sec:problem}).

\subsection{\(n\)-ary Relational Knowledge Graphs}\label{sec:rkg}
Real-world events typically manifest as $n$-ary relational constraints where semantic validity depends on the joint participation of multiple entities and attributes. 

\begin{definition}[\textbf{\(n\)-ary Fact}]
An $n$-ary fact $f$ is an atomic tuple $f = (r, \{v_1, \dots, v_n\})$, where $r \in \mathcal{R}$ is the relation type and $\{v_i\} \subseteq \mathcal{V}$ are the participating entities. Unlike binary triples $(s, r, o)$, $f$ is a \textit{topological unit} whose semantics are lost under binary factorization.
\end{definition}

To preserve these higher-order dependencies, we introduce a \textit{Knowledge Hypergraph} structure.
\begin{definition}[\textbf{Knowledge Hypergraph}]
A knowledge hypergraph is $\mathcal{H} = (\mathcal{V}, \mathcal{E})$, where $\mathcal{V}$ is the set of entity vertices and $\mathcal{E} \subseteq 2^{\mathcal{V}} \setminus \{\emptyset\}$ is the set of hyperedges. Each hyperedge $e \in \mathcal{E}$ represents an $n$-ary fact $f$, preserving event atomicity by connecting all $\{v_i\}$ participants. 
\end{definition}

Each entity $v \in \mathcal{V}$ and hyperedge $e \in \mathcal{E}$ is associated with a textual description $x_v, x_e$, projected into a structure-aware latent space via an encoder to yield embeddings $\mathbf{x}_v, \mathbf{x}_e \in \mathbb{R}^{d_{\text{text}}}$. The collective representations form the textual embedding matrix $\mathbf{X} \in \mathbb{R}^{N \times d_{\text{text}}}$, where $N = |\mathcal{V}| + |\mathcal{E}|$.


\subsection{Sequential Knowledge Editing}\label{sec:ke}
Sequential Knowledge Editing (SKE) involves a continuous stream of update requests $\Delta = \{\delta_1, \dots, \delta_T\}$. Each $\delta_t = (f_{\text{old}} \to f_{\text{new}})$ is a directive to transition the hypergraph from $\mathcal{H}_{t-1}$ to $\mathcal{H}_t$. A successful SKE framework must satisfy:

\begin{definition}[\textbf{Edit Request}]
An edit request $\delta_t$ is a directive to insert, delete, or modify an $n$-ary fact, denoted as $\delta_t = (f_{\text{old}} \to f_{\text{new}})$, where $f_{\text{new}}$ represents the target knowledge state that must be incorporated into the hypergraph topology.
\end{definition}

The objective of Sequential Knowledge Editing (SKE) is to transition the system from state $\mathcal{S}_{t-1}$ to $\mathcal{S}_t$ satisfying three conditions:
\begin{itemize}[leftmargin=*]
    \item \textbf{Reliability (Edit Success):} The system correctly answers queries $Q_{\text{edit}}$ targeting $f_{\text{new}}$, ensuring the new knowledge is effectively integrated \cite{yao2023editing}.
    \item \textbf{Locality (Specificity):} The system maintains performance on queries $Q_{\text{loc}}$ targeting unrelated knowledge, thus mitigating the catastrophic forgetting characteristic of sequential learning \cite{kirkpatrick2017overcoming}.
    \item \textbf{Structural Stability:} The retrieval mechanism adapts to the evolving topology of $\mathcal{H}_t$ to prevent retrieval drift, a challenge frequently overlooked in static retrieval-augmented editing.
\end{itemize}

\subsection{Problem Formulation}\label{sec:problem}
The primary challenge in SKE is \textit{Structural Drift}. As edits accumulate, the topological distribution of $\mathcal{H}_t$ diverges from the retriever's original training prior $\mathcal{D}_0$. This precipitates \textbf{Structure-Conditioned Knowledge Transfer Failure (SKTF)}, where retrieval becomes locally plausible but structurally invalid for $f_{\text{new}}$.


\begin{problem}[\textbf{SKE under \(n\)-ary Structural Drift}]
Given an initial hypergraph $\mathcal{H}_0$, a pre-trained retrieval-augmented reasoning model $\mathcal{M}$, and a continuous stream of $n$-ary edit requests $\Delta = \{\delta_1, \delta_2, \dots, \delta_T\}$, our objective is to learn:
\begin{enumerate}[leftmargin=*]
    \item A \textbf{Structural Graph Update} $\Psi(\mathcal{H}_{t-1}, \delta_t) \to \mathcal{H}_t$ that modifies the hypergraph topology to explicitly incorporate the constraints of $f_{\text{new}}$.
    \item A \textbf{Parameter-Preserving Model Adaptation} $\Phi(\mathcal{M}, \delta_t) \to \mathcal{M}_t$ that aligns the retrieval embedding space with the updated topology without retraining backbone parameters.
\end{enumerate}

\end{problem}

The goal is to maximize multi-hop reasoning accuracy on the evolving query distribution $P_t(q)$ by ensuring that for any $q$, the retrieved hyperedges $\operatorname{TopK}(q, \mathcal{H}_t)$ form a globally consistent event chain satisfying $f_{\text{new}}$.




\section{Methodology}\label{sec:method}
We propose \textbf{\method{}}, a framework designed to resolve \textit{Structure-Conditioned Knowledge Transfer Failure} (SKTF) by enforcing topological consistency during sequential editing. Unlike parametric editors that overwrite weights, \method{} maintains a persistent hypergraph structure, ensuring that \nary{} relations remain semantically atomic across continuous updates.

\subsection{Framework Overview}
As illustrated in Figure~\ref{fig:framework}, the HyperPatch consists of three phases to align the retrieval mechanism with the shifting knowledge topology:
(a) \textbf{Structural Prior Initialization:} We construct hypergraph $\mathcal{H}_0$ and pre-train an HGNN via contrastive learning, establishing a topology-aware embedding space capturing high-order entity correlations.
(b) \textbf{Sequential Topology Editing:} To handle a continuous edit stream $\Delta$, we introduce a dual-stage update mechanism. First, Topological Alignment ($\Psi$) determines whether edit $\delta_t$ requires structural replacement or additive expansion. Second, Topological LoRA Adaptation ($\Phi$) realigns embeddings with modified topology without backbone retraining, tracking structural drift.
(c) \textbf{Structure-Conditioned Reasoning:} During inference, Dual-Manifold Retrieval queries both linguistic and structural manifolds to synthesize globally consistent evidence for multi-hop reasoning.

\subsection{Structural Prior Initialization}\label{sec:gnn_pretrain}
To provide robust structural priors for sequential editing, we pre-train an HGNN on the initial hypergraph $\mathcal{H}_0 = (\mathcal{V}_0, \mathcal{E}_0)$ \cite{feng2019hypergraph}. This phase establishes a topology-aware manifold where the latent representations of entities and $n$-ary facts are aligned.

\noindent \textbf{Event Alignment.}
Let $\mathbf{x}_v \in \mathbb{R}^{d_{\text{text}}}$ and $\mathbf{x}_e \in \mathbb{R}^{d_{\text{text}}}$ denote textual embeddings for entity $v \in \mathcal{V}_0$ and hyperedge $e \in \mathcal{E}_0$ from a frozen encoder. We introduce a learnable projector $P: \mathbb{R}^{d_{\text{text}}} \to \mathbb{R}^{d_{\text{HGNN}}}$ mapping text embeddings to the HGNN's latent space:
\begin{equation}\label{eq:align}
\mathbf{h}_i^{(0)} = P(\mathbf{x}_i), \quad \forall i \in \mathcal{V}_0 \cup \mathcal{E}_0.
\end{equation}

\noindent \textbf{Contrastive Pre-training.}
The HGNN, denoted as $\mathcal{G}_{\theta}$, propagates messages over $\mathcal{H}_0$ to capture high-order dependencies. The pre-training objective enforces ``structural incidence,'' ensuring that entities and the hyperedges they participate in reside closely in the latent space. We optimize $\mathcal{G}_{\theta}$ via a binary cross-entropy contrastive loss. Let $\mathbf{z}_v$ and $\mathbf{z}_e$ be the output embeddings. We define the set of positive incident pairs as $\mathcal{I}^+ = \{(v, e) \mid v \in e\}$ and negative pairs $\mathcal{I}^-$ via random sampling:
\begin{equation}
\mathcal{L}_{\text{s}} = \frac{-1}{|\mathcal{I}^+|} \sum_{(v,e) \in \mathcal{I}^+} \left[ \log \sigma(\mathbf{z}_v^\top \mathbf{z}_e) + \mathbb{E}_{e' \sim \mathcal{I}^-} \left[ \log (1 - \sigma(\mathbf{z}_v^\top \mathbf{z}_{e'})) \right] \right],
\end{equation}
where $\sigma(\cdot)$ is the sigmoid function\cite{ijcai2025p629, liao2024hypergraph, chen2024hyperedge}.

\noindent \textbf{Parameter Preservation.} 
During initialization, gradients flow through both $P$ and $\mathcal{G}_{\theta}$. Upon convergence, $\mathcal{G}_{\theta}$ is frozen to serve as a structural backbone. This allows the subsequent adaptation phase ($\Phi$) to leverage topological low-rank adaptation (LoRA) modules, preventing the collapse of the global topological prior while specializing in local updates.

\subsection{Sequential Topology Editing}\label{sec:ste}
To accommodate continuous updates, \method{} introduces a dual-stage structural adaptation process that identifies topological discrepancies via hashing and aligns semantic components via linguishtic span extraction.

\subsubsection{\textbf{Efficient Hyperedge Linking}}
To enable the structural update $\Psi(\mathcal{H}_{t-1}, \delta_t)$, \method{} must rapidly identify the hyperedge $e \in \mathcal{E}_{t-1}$ corresponding to the obsolete fact $f_{\text{old}}$. We employ a SimHash-based \cite{manku2007detecting} locality-sensitive hashing (LSH) approach to achieve $O(1)$ approximate nearest-neighbor search.

Specifically, we decompose a hyperedge's textual representation into $M$ 4-gram tokens $\mathcal{T} = \{T_1, \dots, T_M\}$. Each token $T_m$ is mapped to a $K$-bit binary vector $\mathbf{B}_m \in \{0,1\}^K$ via a deterministic MD5 hash function $F_{\text{MD5}}$.
As a result, the set of hashed binary vectors is given by 
\(\mathcal{B} = \{\mathbf{B}_1, \dots, \mathbf{B}_{M}\}\),
where each \(\mathbf{B}_m\) corresponds to the \(m\)-th 4-gram token.
To aggregate these set of hashed binary vectors into a single structural fingerprint, each $\mathbf{B}_m$ is transformed into signed vectors $\mathbf{S}_m \in \{-1, +1\}^K$, where $S_m^{(k)} = 2B_m^{(k)} - 1$. 
The structural fingerprint $\mathbf{b} \in \{0,1\}^K$ is generated by aggregating these vectors:
\begin{equation}
    \mathbf{b}^{(k)} = \mathbb{I}\left(\sum\nolimits_{m=1}^{M} S_m^{(k)} > 0\right).
\end{equation}
To ensure retrieval scalability, fingerprints are indexed using a \texttt{FAISSBinaryFlat} index \cite{douze2025faiss}: $\texttt{Index} \leftarrow \texttt{FAISSBinaryFlat}(K)$. This mechanism allows \method{} to locate $f_{\text{old}}$ within $\mathcal{H}_{t-1}$ with minimal latency, effectively pruning the search space for large-scale structural alignment.

\subsubsection{\textbf{Entity Span Extraction for Knowledge Alignment}}
While hashing facilitates rapid retrieval, precise alignment requires a granular decomposition of $n$-ary facts. We define a span extraction function $\mathcal{S}(e)$ that maps a hyperedge $e$ to an ordered sequence of semantic anchors $\{s_1, \dots, s_k\}$. To ensure compound entities are treated as atomic nodes, tokens are aggregated into a span $s_i$ if they belong to a predefined set of Part-of-Speech (POS) tags $\mathcal{P} = \{\texttt{NOUN, PROPN, NUM, ADJ, PRON}\}$, dependency roles $\mathcal{D} = \{\texttt{det, amod, pcomp, prep, advmod}\}$, or relational connectors $\mathcal{A} = \{\text{``of'', ``'s'', ``and'', ``-''}\}$. 

The resulting sequence $\mathcal{S}(e)$ provides a formal basis for quantifying the structural discrepancy between $f_{\text{old}}$ and $f_{\text{new}}$. By analyzing these sequences, \method{} characterizes the precise nature of the topological shift—distinguishing among attribute modifications, entity replacements, and relational expansions. This characterization subsequently steers both the topological modification $\Psi$ and the parameter-efficient embedding adaptation $\Phi$, ensuring the hypergraph accurately reflects the updated knowledge state.

\subsubsection{\textbf{Topological Conflict Resolution}}\label{sec:conflict-res}
To automate structural update $\Psi(\mathcal{H}_{t-1}, \delta_t)$, \method{} classifies each edit request into one of three topological operations\cite{zhang2025conflictaware}: \textsc{Replace}, \textsc{Add}, or \textsc{Ignore}. This classification uses structural divergence between extracted entity spans of $f_{\text{new}}$ and candidate hyperedges from SimHash indexing.

For each top-$k$ candidate $e_{\text{old}} \in \mathcal{E}_{t-1}$ retrieved via Hamming distance, we compute asymmetric span discrepancies:
\begin{equation}
\Delta_{\text{old}} = \mathcal{S}(e_{\text{old}}) \setminus \mathcal{S}(e_{\text{new}}), \quad 
\Delta_{\text{new}} = \mathcal{S}(e_{\text{new}}) \setminus \mathcal{S}(e_{\text{old}}).
\end{equation}

\noindent \textbf{Replace Operation. }
A candidate $e_{\text{old}}$ is designated as a \textsc{Replace} target if and only if $|\Delta_{\text{old}}| = |\Delta_{\text{new}}| = 1$, and and both discordant spans are non-subject components. Subject entities define event identity, while non-subject modifications represent state transitions. We define heuristic subject-detection $\phi(\cdot) : \mathcal{S} \rightarrow \{0,1\}$ based on declarative structure where subjects precede predicates:
\begin{equation}
\phi(s_i) =
\begin{cases}
1, & i < \lceil \frac{|\mathcal{S}(e)|}{2} \rceil \\
0, & \text{otherwise}
\end{cases}.
\end{equation}
Let $\Delta_{\text{old}} = \{s_o\}$ and $\Delta_{\text{new}} = \{s_n\}$. A \textsc{Replace} operation is executed if $\phi(s_o) = 0 \wedge \phi(s_n) = 0$. Formally, the hyperedge set is updated as $\mathcal{E}_t = (\mathcal{E}_{t-1} \setminus \{e_{\text{old}}\}) \cup \{e_{\text{new}}\}$, overwriting the obsolete fact while maintaining the structural vertex incidence of the subject.

\noindent \textbf{Add Operation.} Conversely, if the subject identity is modified or the structural discrepancy remains bounded (e.g., $|\Delta_{\text{old}}| = |\Delta_{\text{new}}| = 2$), \method{} interprets the request as a semantic expansion or a distinct event. In this case, an \textsc{Add} operation is executed: $\mathcal{E}_t = \mathcal{E}_{t-1} \cup \{e_{\text{new}}\}$. This bifurcated logic ensures that $\mathcal{H}_t$ faithfully incorporates the target state $f_{\text{new}}$ while preserving the structural integrity of the global knowledge base.

\noindent \textbf{Ignore Condition.} The Ignore operation serves as a safety guardrail for ambiguous or unsupported edit requests. It is triggered when an edit request satisfies neither the Replace criteria, e.g., a valid single-span non-subject substitution, nor the bounded Add criteria, e.g., a structurally related expansion with limited span discrepancy. In particular, requests with mismatched event identity or large unbalanced discrepancies, such as $\max(|\Delta_{\text{old}}|, |\Delta_{\text{new}}|) > 2$, are treated as unreliable structural signals. In these ambiguous scenarios, HyperPatch favors omission over potentially corruptive structural changes, thereby preserving the semantic consistency of the evolving hypergraph.

\subsubsection{\textbf{Incremental Topological Adaptation}}
Following the structural update, we perform model adaptation $\Phi(\mathcal{M}, \delta_t)$ by incrementally updating the hypergraph embeddings $\mathbf{Z}_t$ using a parameter-efficient LoRA-based HGNN~\cite{yang2025graphlora, chen2025decoupling}. This ensures that the retrieval manifold remains synchronized with the evolved topology $\mathcal{H}_t$, thereby mitigating retrieval drift.

Specifically, we utilize the pre-trained HGNN encoder as a frozen structural backbone to preserve global regularities while allowing gradients to propagate through a learnable projector $\boldsymbol{\theta}_{\text{proj}}$ and injected low-rank adapters $\boldsymbol{\theta}_{\text{LoRA}}$. Let $f_{\text{HGNN}}(\mathcal{H}_t, \mathbf{X}_t; \boldsymbol{\theta})$ denote the encoder with parameters $\boldsymbol{\theta} = \boldsymbol{\theta}_{\text{base}} \cup \boldsymbol{\theta}_{\text{proj}} \cup \boldsymbol{\theta}_{\text{LoRA}}$. The adaptation is framed as a constrained optimization task:
\begin{equation}
\min_{\boldsymbol{\theta}_{\text{proj}}, \boldsymbol{\theta}_{\text{LoRA}}} \mathcal{L}_{\text{struct}}\left(f_{\text{HGNN}}(\mathcal{H}_t, \mathbf{X}_t; \boldsymbol{\theta}_{\text{base}}, \boldsymbol{\theta}_{\text{proj}}, \boldsymbol{\theta}_{\text{LoRA}})\right),
\end{equation}
where $\boldsymbol{\theta}_{\text{base}}$ remains frozen. The updated embeddings $\mathbf{Z}_t$ are then cached to facilitate real-time, structure-aware retrieval without the prohibitive cost of full backbone retraining.

\subsection{Structure-Conditioned Reasoning}\label{sec:rag}
The final phase of \method{} utilizes the updated hypergraph $\mathcal{H}_t$ and the adapted model $\mathcal{M}_t$ to perform multi-hop reasoning. We explicitly address \textit{Structural Drift} by ensuring the reasoning process is grounded in both linguistic semantics and the evolved topology.

\begin{table*}[!ht] \centering 
\setlength{\tabcolsep}{2pt} 
\definecolor{hl}{gray}{0.95} 
\definecolor{tu}{RGB}{255, 248, 230} 
\definecolor{pr}{RGB}{235, 245, 255} 
\definecolor{gaintext}{RGB}{0, 80, 160} 

\resizebox{0.7\textwidth}{!}{
\begin{tabular}{ll cc cc cc cc cc}
\toprule
\multirow{2}{*}{\textbf{Backbone}} & \multirow{2}{*}{\textbf{Method}} & \multicolumn{6}{c}{\textbf{MQuAKE-CF-3K-v2 (Counterfactual)}} & \multicolumn{4}{c}{\textbf{MQuAKE-T (Temporal)}} \\
\cmidrule(lr){3-8} \cmidrule(lr){9-12}
& &  \multicolumn{2}{c}{1 Edit} & \multicolumn{2}{c}{100 Edits} & \multicolumn{2}{c}{All Edits} & \multicolumn{2}{c}{1 Edit} & \multicolumn{2}{c}{All Edits} \\
\cmidrule(lr){3-4} \cmidrule(lr){5-6} \cmidrule(lr){7-8} \cmidrule(lr){9-10} \cmidrule(lr){11-12}
& &  M-A & H-A & M-A & H-A & M-A & H-A & M-A & H-A & M-A & H-A \\
\midrule
\multirow{8}{*}{\shortstack{Qwen3-8B\\(Open)}}  
& \cellcolor{tu}FT \cite{zhu2020modifying} & 26.37 & 6.80 & 14.60 & 1.97 & 14.50 & 2.30 & 26.98 & 12.96 & 15.10 & 9.05 \\
& \cellcolor{tu}ROME \cite{meng2022locating} & 34.73 & 8.33 & 15.03 & 1.10 & 11.07 & 0.20 & 40.15 & 15.46 & 36.88 & 15.58 \\
& \cellcolor{tu}MEMIT \cite{meng2023mass} & 29.20 & 5.96 & 20.67 & 3.30 & 1.60 & 0.03 & 45.45 & 16.76 & 55.23 & 21.15 \\
\cmidrule(lr){2-12}
&  \cellcolor{pr}MeLLo \cite{zhong2023mquake} & \underline{51.93} & 35.53 & 46.37 & 37.23 & 43.33 & 36.03 & 84.42 & 56.53 & 83.73 & 65.52 \\
&  \cellcolor{pr}PokeMQA \cite{gu2024pokemqa} & 51.43 & \underline{42.87} & 46.43 & \underline{40.97} & 41.73 & \underline{36.40} & 77.30 & 68.36 & 75.75 & 67.83 \\
&  \cellcolor{pr}KeDKG \cite{lu2025knowledge} & 47.67 & 39.70 & \underline{47.07} & 38.80 & \underline{44.33} & 35.63 & \underline{85.55} & \underline{78.91} & \underline{84.21} & \underline{77.30} \\
& \cellcolor{pr}\textbf{\method{}} & \cellcolor{hl}\textbf{94.70} & \cellcolor{hl}\textbf{73.60} & \cellcolor{hl}\textbf{81.33} & \cellcolor{hl}\textbf{74.33} & \cellcolor{hl}\textbf{76.20} & \cellcolor{hl}\textbf{71.43} & \cellcolor{hl}\textbf{99.30} & \cellcolor{hl}\textbf{93.63} & \cellcolor{hl}\textbf{96.84} & \cellcolor{hl}\textbf{93.58}\\ 
\cmidrule(lr){2-12}
& \textbf{\color{gaintext}Gain (\%)} & \small\color{gaintext}\textbf{+82.4} & \small\color{gaintext}\textbf{+71.7} & \small\color{gaintext}\textbf{+72.8} & \small\color{gaintext}\textbf{+81.4} & \small\color{gaintext}\textbf{+71.9} & \small\color{gaintext}\textbf{+96.2} & \small\color{gaintext}\textbf{+16.1} & \small\color{gaintext}\textbf{+18.7} & \small\color{gaintext}\textbf{+15.0} & \small\color{gaintext}\textbf{+21.1} \\
\midrule
\multirow{5}{*}{\shortstack{GPT-4o-mini\\(Prop.)}}
& \cellcolor{pr}MeLLo \cite{zhong2023mquake} & 43.77 & 18.70 & 37.40 & 18.67 & 35.33 & 17.83 & 82.82 & 56.26 & 57.87 & 36.40 \\
& \cellcolor{pr}PokeMQA \cite{gu2024pokemqa} & \underline{50.07} & 38.17 & 46.37 & 35.87 & 40.27 & 31.53 & 87.04 & 74.68 & 84.42 & 72.54 \\
& \cellcolor{pr}KeDKG \cite{lu2025knowledge} & 48.13 & \underline{41.47} & \underline{47.13} & \underline{39.47} & \underline{46.37} & \underline{36.77} & \underline{88.17} & \underline{83.24} & \underline{86.67} & \underline{81.26} \\
& \cellcolor{pr}\textbf{\method{}} & \cellcolor{hl}\textbf{93.90} & \cellcolor{hl}\textbf{74.63} & \cellcolor{hl}\textbf{81.23} & \cellcolor{hl}\textbf{74.27} & \cellcolor{hl}\textbf{74.63} & \cellcolor{hl}\textbf{69.57} & \cellcolor{hl}\textbf{99.20} & \cellcolor{hl}\textbf{94.16} & \cellcolor{hl}\textbf{96.41} & \cellcolor{hl}\textbf{93.63} \\
\cmidrule(lr){2-12}
 & \textbf{\color{gaintext}Gain (\%)} & \small\color{gaintext}\textbf{+87.5} & \small\color{gaintext}\textbf{+80.0} & \small\color{gaintext}\textbf{+72.4} & \small\color{gaintext}\textbf{+88.2} & \small\color{gaintext}\textbf{+60.9} & \small\color{gaintext}\textbf{+89.2} & \small\color{gaintext}\textbf{+12.5} & \small\color{gaintext}\textbf{+13.1} & \small\color{gaintext}\textbf{+11.2} & \small\color{gaintext}\textbf{+15.2} \\
\bottomrule
\end{tabular}
}
\caption{Knowledge editing performance on MQuAKE-CF-3K and MQuAKE-T. Parameter-preserving methods are applicable to both proprietary and open-source models, whereas parameter-tuning baselines are restricted to open-source settings (metrics in \%). Gain (\%) denotes the relative gain over the strongest baseline. The best results are bolded, and the second best are underlined.}
\vspace{-4mm}
\label{tab:main}
\end{table*}

\subsubsection{\textbf{Iterative Question Decomposition}}
To navigate the $n$-ary relational structure of $\mathcal{H}_t$, we decompose a complex query $q^{(0)}$ into an ordered sequence of atomic sub-questions $\mathcal{Q} = \{ q^{(1)}, \dots, q^{(I)} \}$ using an LLM-based planner\cite{ammann-etal-2025-question}. To maintain global consistency, we implement a state-aware injection mechanism where each subsequent sub-question is conditioned on the preceding intermediate answer $a^{(i-1)}$:
\begin{equation}
q^{(i)} \leftarrow \text{Inject}(q^{(i)}, \text{\texttt{[ENT]}}, a^{(i-1)}).
\end{equation}
Each $q^{(i)}$ is answered by $\mathcal{M}_t$, yielding a reasoning chain $\text{A} = \{ a^{(1)}, \dots, a^{(I)} \}$. This iterative refinement transforms multi-hop reasoning into a sequence of dependent RAG steps, each conditioned on the topological neighborhood of the previous hop’s result.

\subsubsection{\textbf{Dual-Manifold Retrieval}}
Standard retrievers often suffer from \textit{Structure-Conditioned Knowledge Transfer Failure} (SKTF) because they rely exclusively on pre-trained linguistic similarity, which fails to reflect sequential topological updates. \method{} mitigates this by querying two distinct latent manifolds: the \textit{Semantic Manifold} ($\mathcal{Z}_{\text{text}}$) and the \textit{Topological Manifold} ($\mathcal{Z}_{\text{HGNN}}$)\cite{li-etal-2025-tailoring}.

\noindent \textbf{Manifold Alignment.} 
For a sub-query $q^{(i)}$, we generate a linguistic embedding $\mathbf{x}_{\text{text}} \in \mathbb{R}^{d_{\text{text}}}$ using the backbone encoder. To synchronize this with the updated topology, we project $\mathbf{x}_{\text{text}}$ into the structure-aware latent space via the frozen projector $P$ (Eq.~\eqref{eq:align}):
\begin{equation}
\mathbf{x}_{\text{HGNN}} = P(\mathbf{x}_{\text{text}}) \in \mathbb{R}^{d_{\text{HGNN}}}.
\end{equation}

\noindent \textbf{Joint Evidence Retrieval.}  
\method{} performs parallel Maximum Inner Product Search (MIPS) across both manifolds. For a sub-query $q^{(i)}$ at hop $i$, the candidate sets are retrieved as:
\begin{equation}
\begin{aligned}
\mathcal{C}_{\text{text}} &= \operatorname{TopK}_{j \in \{\mathcal{V}_t \cup \mathcal{E}_t\}} \left( \frac{\mathbf{x}_{\text{text}}^\top \mathbf{z}_{\text{text}, j}}{\|\mathbf{x}_{\text{text}}\| \|\mathbf{z}_{\text{text}, j}\|} \right), \\
\mathcal{C}_{\text{HGNN}} &= \operatorname{TopK}_{j \in \{\mathcal{V}_t \cup \mathcal{E}_t\}} \left( \frac{\mathbf{x}_{\text{HGNN}}^\top \mathbf{z}_{\text{HGNN}, j}}{\|\mathbf{x}_{\text{HGNN}}\| \|\mathbf{z}_{\text{HGNN}, j}\|} \right),
\end{aligned}
\end{equation}
where $\mathcal{V}_t \cup \mathcal{E}_t$ denotes the updated universe of entities and hyperedges. This dual-manifold aggregation identifies evidence that is both linguistically salient and topologically congruent. The final retrieved context is the union: $\mathcal{C}^{(i)} = \mathcal{C}_{\text{text}} \cup \mathcal{C}_{\text{HGNN}}$. This ensures a balanced context that preserves semantic intent while respecting structural constraints imposed by the edit stream $\Delta$.

\noindent \textbf{Reasoning Synthesis.}
Following retrieval, \method{} merge hyperedges and de-duplicate overlapping spans to ensure factual density. This grounded context is then provided to the generator $G$ to predict the intermediate answer $a^{(i)}$. This mechanism specifically counters SKTF: even if linguistic embeddings remain biased toward the initial prior $\mathcal{D}_0$, the structure-aware embeddings, which are adapted via $\Phi$, provide the necessary corrective signal to track the evolved topology of $\mathcal{H}_t$.

\subsection{Theoretical Complexity Bound}\label{sec:bigo}
Sequential KE demands update latency independent of $|\mathcal{E}|$. We establish HyperPatch's theoretical efficiency guarantee.

\begin{theorem}[\textbf{Editing Efficiency}]
Let $K$ denote the SimHash fingerprint length, $|T|$ the $n$-gram tokens per edit, $d$ the embedding dimension, and $r$ the LoRA rank with $r \ll d$. The amortized complexity per edit is $O(K \cdot |T| + d \cdot r)$, independent of $|\mathcal{E}|$.
\end{theorem} 

\begin{proof}[Proof Sketch]
The editing phase comprises two coupled operations.
\textbf{(i)~SimHash Linking}: Computing the $K$-bit fingerprint from $|T|$ $n$-grams via MD5 hashing and majority voting requires $O(K \cdot |T|)$. LSH-based retrieval from \texttt{FAISSBinaryFlat} is $O(1)$ amortized w.r.t.\ $|\mathcal{E}|$, yielding constant-time conflict resolution. 
\textbf{(ii)~LoRA Adaptation}: Updating $\mathbf{W} + \mathbf{AB}^\top$ where $\mathbf{A} \in \mathbb{R}^{d \times r}$, $\mathbf{B} \in \mathbb{R}^{r \times d}$ reduces complexity from $O(d^2)$ to $O(d \cdot r)$ since $r \ll d$. The synergy is critical as SimHash localizes updates to $O(1)$ hyperedges, while LoRA ensures each update costs $O(d \cdot r)$ not $O(d^2)$, yielding $O(K \cdot |T| + d \cdot r)$ independent of $|\mathcal{E}|$. This achieves speedups of $\frac{d}{r}$ over parametric methods and $\frac{|\mathcal{E}|}{\log |\mathcal{E}|}$ over exhaustive retrieval, ensuring scalability to millions of hyperedges and enabling real-time editing.
\end{proof}

\begin{theorem}[\textbf{Amortized Reasoning Scalability}]
Let $K$ be the SimHash fingerprint length, $|T|$ the tokens per edit, $d$ the embedding dimension, $r$ the LoRA rank ($r \ll d$), and $I$ the number of reasoning hops. The cost per multi-hop query is $O(I \cdot (d^2 + \log |\mathcal{H}| \cdot d))$.
\end{theorem} 

\begin{proof}[Proof Sketch]
For each reasoning hop $i \in \{1 \dots I\}$, manifold alignment via projector $P$ incurs $O(d^2)$ cost. Evidence retrieval utilizes dual-manifold MIPS; with HNSW indexing, retrieval complexity is $O(\log |\mathcal{H}| \cdot d)$. Thus, the total reasoning cost is dominated by the reasoning depth $I$ and latent dimensionality $d$, scaling logarithmically with the knowledge base size. This enables \method{} to handle millions of \nary relations with real-time performance.
\end{proof}
Both phases are sub-linear or independent of the total facts $|\mathcal{H}|$.

\section{Experiments}\label{sec:exp}

\subsection{Experimental Setup}



\noindent \textbf{Datasets. }
We evaluate \method{} on the \textsf{MQuAKE} benchmark \cite{zhong2023mquake}, using two widely-used non-overlapping subsets: \textbf{MQuAKE-CF-3k-v2} and \textbf{MQuAKE-T}. MQuAKE-CF-3k-v2 is a 3,000-instance counterfactual benchmark designed to resolve reasoning ambiguities in earlier versions, while MQuAKE-T contains 1,868 real-world temporal editing instances. Unlike temporal knowledge graph benchmarks that provide explicit timestamps or validity intervals, MQuAKE-T represents temporal evolution as discrete factual state transitions, where an outdated fact is replaced by its updated counterpart. Therefore, these temporal changes can be naturally modeled as structural Replace/Add operations over atomic \nary{} events, without requiring a separate temporal reasoning branch. As shown in Table~\ref{tab:MQuAKE distribution}, MQuAKE-CF-3K-v2 is evenly distributed across 2-hop, 3-hop, and 4-hop questions. Most 2-hop questions involve one or two edits, whereas 4-hop questions generally require more edits, reflecting the increased complexity of deeper reasoning chains. In contrast, MQuAKE-T is dominated by 2-hop and 3-hop questions, with only two 4-hop cases. Together, these two subsets provide a rigorous testbed for evaluating sequential \nary{} editing under both counterfactual and temporal knowledge updates across different reasoning depths.

\begin{table}[!h]
\centering
\small
\setlength{\tabcolsep}{3pt}
\begin{tabular}{lrrrrr}
\toprule
\textbf{Datasets} & \textbf{\#Edits} & \textbf{2-hop} & \textbf{3-hop} & \textbf{4-hop} & \textbf{Total} \\
\midrule
& 1 & 599  & 423  & 51   & 1073 \\
& 2 & 536  & 374  & 136  & 1046 \\
MQuAKE-CF-3K-v2 &3 & --   & 339  & 229  & 568   \\
& 4 & --   & --   & 313  & 313   \\
& All & 1135 & 1136 & 729 & 3000 \\
\midrule
MQuAKE-T &1(All) & 1421   & 445  & 2  & 1868  \\
\bottomrule
\end{tabular}
\caption{Distribution of multi-hop questions by number of edits and hop count in  MQuAKE-CF-3K-v and MQuAKE-T.}
\label{tab:MQuAKE distribution}
\end{table}

\par\noindent \textbf{Evaluation Metrics. }
Following \cite{gu2024pokemqa}, we employ Multi-hop Accuracy \textbf{(M-Acc)} and Hop-wise Answering Accuracy \textbf{(H-Acc)}. M-Acc evaluates final answer correctness, whereas H-Acc requires every intermediate reasoning step to be correct, serving as our primary metric by mitigating coincidental successes from flawed reasoning.

\par\noindent \textbf{Baselines. }
We compare \method{} with two categories of representative knowledge editing approaches. 
(i) The \textit{parameter-preserving} baselines include MeLLo~\cite{zhong2023mquake}, PokeMQA~\cite{gu2024pokemqa}, and KeDKG~\cite{lu2025knowledge}, which edit knowledge without modifying model parameters. 
(ii) In contrast, the \textit{parameter-based methods} involve fine-tuning or memory injection, such as FT~\cite{zhu2020modifying}, ROME~\cite{meng2022locating}, and MEMIT~\cite{meng2023mass}.


\par\noindent \textbf{Implementation Details. }
We utilize GAT~\cite{velickovic2018graph} as the HGNN backbone, projecting 1536-dimensional text embeddings into a 256-dimensional latent manifold. Structural retrieval leverages 128-bit SimHash signatures for Hamming-distance approximate nearest neighbor (ANN) search. Top-$k$ parameters for hyperedge linking are set to 2/70/1200 (CF-3k-v2: 1/100/All edits) and 1/23 (T: 1/All edits). Linguistic parsing is handled via spaCy's \texttt{en\_core\_web\_lg}. We employ \texttt{GPT-4o-mini} for hypergraph construction and question decomposition, with \texttt{GPT-4o-mini} and \texttt{Qwen3-8B} as reasoning backbones. Experiments are executed on a single NVIDIA H100-NVL-94G GPU. 

\subsection{Main Results}\label{sec:rq1}
Table~\ref{tab:main} summarizes the performance of \method{} compared to state-of-the-art baselines. We observe several key findings:

\par\noindent\textbf{Superior Reasoning Stability and Scalability.} 
 \method{} consistently exceeds all baselines in Multi-hop Accuracy (M-Acc) and the more stringent Hop-wise Accuracy (H-Acc). On \textsf{MQuAKE-CF-3K-v2} (Qwen3-8B), \method{} yields relative H-Acc gains of \textbf{79.96\%}, \textbf{88.17\%}, and \textbf{89.20\%} across 1-, 100-, and all-edit settings, respectively. For GPT-4o-mini, relative improvement reaches \textbf{96.24\%} in the all-edit scenario. While baseline performance collapses as edit volume increases, \method{} maintains structural stability, demonstrating robustness under dense update streams.

\par\noindent\textbf{Mitigation of Event Factorization Drift.} The performance gap exposes two failure modes in current editors: (i) \textit{Localization Error}, where inaccurate update anchoring induces internal contradictions; and (ii) \textit{Retrieval Misalignment}, where triple-based methods (e.g., KeDKG) suffer from \textit{Event Factorization Drift}, i.e., the binary reification of \nary{} facts yielding incomplete subgraphs. This forces the generator to rely on stale parametric priors rather than updated evidence. Conversely, \method{} preserves relational atomicity via its hypergraph manifold, identifying globally consistent event structures despite significant drift.

\par\noindent\textbf{Bridging the Model Disparity.} While proprietary models typically excel in instruction-following, \method{} empowers open-source models (e.g., Qwen3-8B) to achieve reasoning parity with closed-source counterparts. Notably, Qwen3-8B equipped with \method{} occasionally outperforms GPT-4o-mini-based baselines. This suggests that structural alignment effectively decouples reasoning efficacy from model scale, enabling smaller open-source models to compete in knowledge-intensive multi-hop tasks.

\par\noindent\textbf{Non-Parametric Robustness vs. Parametric Interference.} 
To assess weight modification risks, we compared \method{} against parametric editors (FT, ROME, MEMIT). Even with oracle sub-queries to minimize reasoning noise, these methods failed to maintain consistency. Direct parameter overwriting frequently induces \textit{catastrophic interference}, disrupting pre-trained reasoning paths, which is a degradation that intensifies with cumulative updates. \method{}'s non-parametric approach sidesteps the plasticity-stability trade-off by offloading knowledge storage to a stable hypergraph substrate, thereby preserving reasoning integrity without compromising internal model logic.

\begin{table*}[t]
\centering
\small
\setlength{\tabcolsep}{1.8pt} 

\newcommand{\down}[1]{\, \scriptsize\textcolor{red}{$\downarrow$#1\%}}
\newcommand{\up}[1]{\, \scriptsize\textcolor{blue}{$\uparrow$#1\%}}
\definecolor{fullbg}{gray}{0.95}
\definecolor{qdlbg}{rgb}{0.94, 0.97, 1.0} 

\resizebox{0.98\textwidth}{!}{
\begin{tabular}{l | cc cc cc | cc cc}
\toprule
\multirow{2}{*}{\textbf{Ablation Setting}} & \multicolumn{6}{c|}{\textbf{\textsf{MQuAKE-CF-3K-v2} (Counterfactual)}} & \multicolumn{4}{c}{\textbf{\textsf{MQuAKE-T} (Temporal)}} \\
\cmidrule(lr){2-7} \cmidrule(lr){8-11}
& \multicolumn{2}{c}{1 Edit} & \multicolumn{2}{c}{100 Edits} & \multicolumn{2}{c|}{All Edits} & \multicolumn{2}{c}{1 Edit} & \multicolumn{2}{c}{All Edits} \\
\cmidrule(lr){2-3} \cmidrule(lr){4-5} \cmidrule(lr){6-7} \cmidrule(lr){8-9} \cmidrule(lr){10-11}
& M-Acc & H-Acc & M-Acc & H-Acc & M-Acc & H-Acc & M-Acc & H-Acc & M-Acc & H-Acc \\
\midrule

w/o Topological Manifold & 85.77 \down{8.7} & 70.40 \down{5.7} & 75.93 \down{6.5} & 67.63 \down{9.0} & 69.63 \down{6.7} & 64.27 \down{7.6} & 96.36 \down{2.9} & 90.69 \down{3.7} & 64.00 \down{33.6} & 89.83 \down{4.1} \\

w/o Conflict Resolution & \textbf{98.63} \up{5.0} & 66.43 \down{11.0} & 61.20 \down{24.7} & 53.20 \down{28.4} & 36.43 \down{51.2} & 32.20 \down{53.7} & \underline{99.10} \down{0.1} & \underline{93.84} \down{0.3} & 86.56 \down{10.2} & 82.60 \down{11.8} \\

w/o SimHash Indexing & 77.87 \down{17.1} & 57.90 \down{22.4} & 24.30 \down{70.1} & 18.80 \down{74.7} & 23.83 \down{68.1} & 17.80 \down{74.4} & 75.75 \down{23.6} & 71.31 \down{24.3} & 60.71 \down{37.0} & 58.24 \down{37.8} \\

w/o LoRA Adaptation & \underline{94.47} \up{0.6} & \underline{74.07} \down{0.8} & \underline{80.70} \down{0.7} & \underline{73.43} \down{1.1} & \underline{73.47} \down{1.6} & \underline{68.53} \down{1.5} & 99.04 \down{0.2} & 93.42 \down{0.8} & \underline{96.20} \down{0.2} & \underline{93.47} \down{0.2} \\

w/ Standard KG (Triples) & 75.07 \down{20.1} & 52.20 \down{30.1} & 52.20 \down{35.7} & 43.53 \down{41.4} & 46.20 \down{38.1} & 39.70 \down{42.9} & 79.18 \down{20.2} & 72.86 \down{22.6} & 15.85 \down{83.6} & 10.97 \down{88.3} \\

\midrule
\rowcolor{qdlbg} KeDKG w/ Llama-2-7B & 37.00 \down{60.6} & 29.33 \down{60.7} & 36.70 \down{54.8} & 28.33 \down{61.9} & 34.20 \down{54.2} & 27.47 \down{60.5} & 84.05 \down{15.3} & 71.47 \down{24.1} & 83.30 \down{13.6} & 70.02 \down{25.2} \\
\rowcolor{qdlbg} KeDKG w/ GPT-4o-mini & 48.13 \down{48.7} & 41.47 \down{44.4} & 47.13 \down{42.0} & 39.47 \down{46.9} & 46.37 \down{37.9} & 36.77 \down{47.1} & 88.17 \down{11.1} & 83.24 \down{11.6} & 86.67 \down{10.1} & 81.26 \down{13.2} \\
\rowcolor{qdlbg} HyperPatch w/ Llama-2-7B & 69.10 \down{26.4} & 37.53 \down{49.7} & 46.57 \down{42.7} & 37.40 \down{49.6} & 41.50 \down{44.4} & 36.40 \down{47.7} & 96.25 \down{3.0} & 75.96 \down{19.3} & 89.08 \down{7.6} & 75.80 \down{19.0} \\

\midrule
\rowcolor{fullbg} \textbf{\method{} (Full)} & 93.90 & \textbf{74.63} & \textbf{81.23} & \textbf{74.27} & \textbf{74.63} & \textbf{69.57} & \textbf{99.20} & \textbf{94.16} & \textbf{96.41} & \textbf{93.63} \\
\bottomrule
\end{tabular}
}
\caption{Ablation study quantifying the relative contribution of individual components. Colored indicators denote the relative percentage change compared to the full \method{} model. Best results are \textbf{bolded}; second-best are \underline{underlined} (metrics in \%).}
\label{tab:ablation}
\vspace{-4mm}
\end{table*}

\begin{figure}[t]
    \centering
    \includegraphics[width=0.99\linewidth]{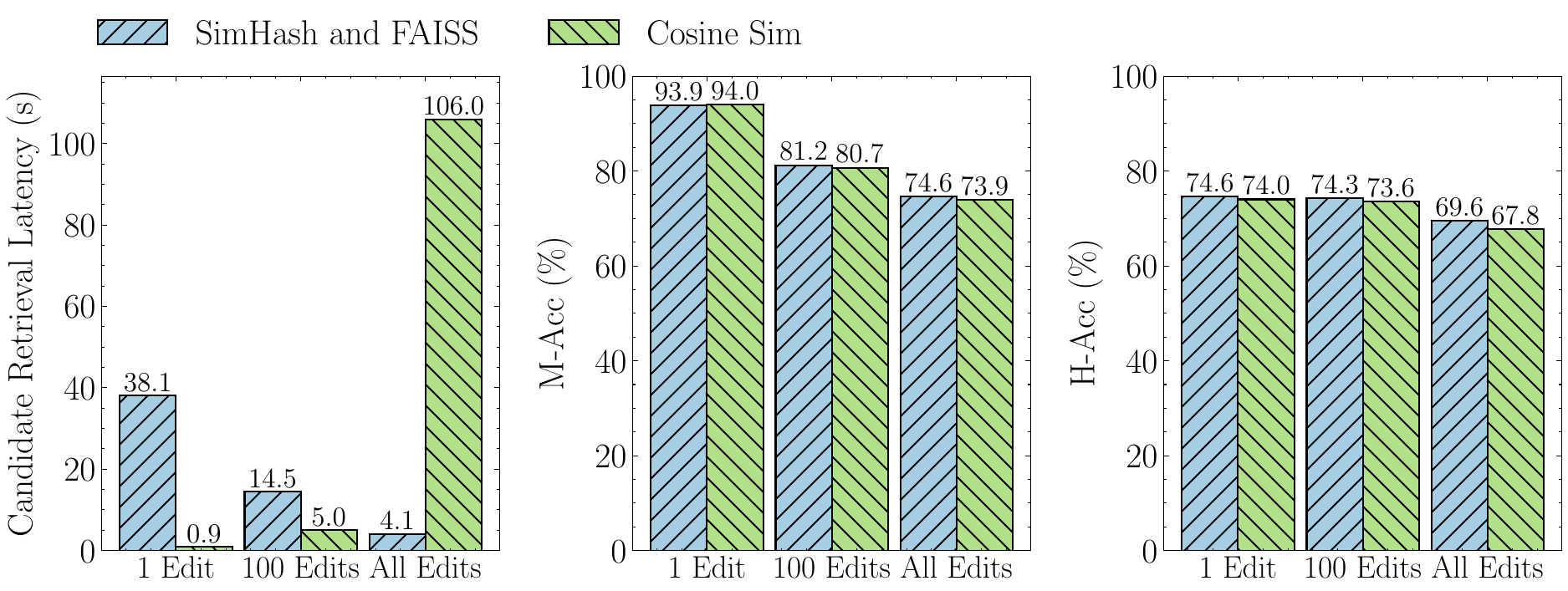}
    \caption{Candidate Retrieval: Efficiency vs. Effectiveness.}
    \label{fig:efficiency}
\end{figure}

\begin{figure}[t]
    \centering
    \includegraphics[width=0.99\linewidth]{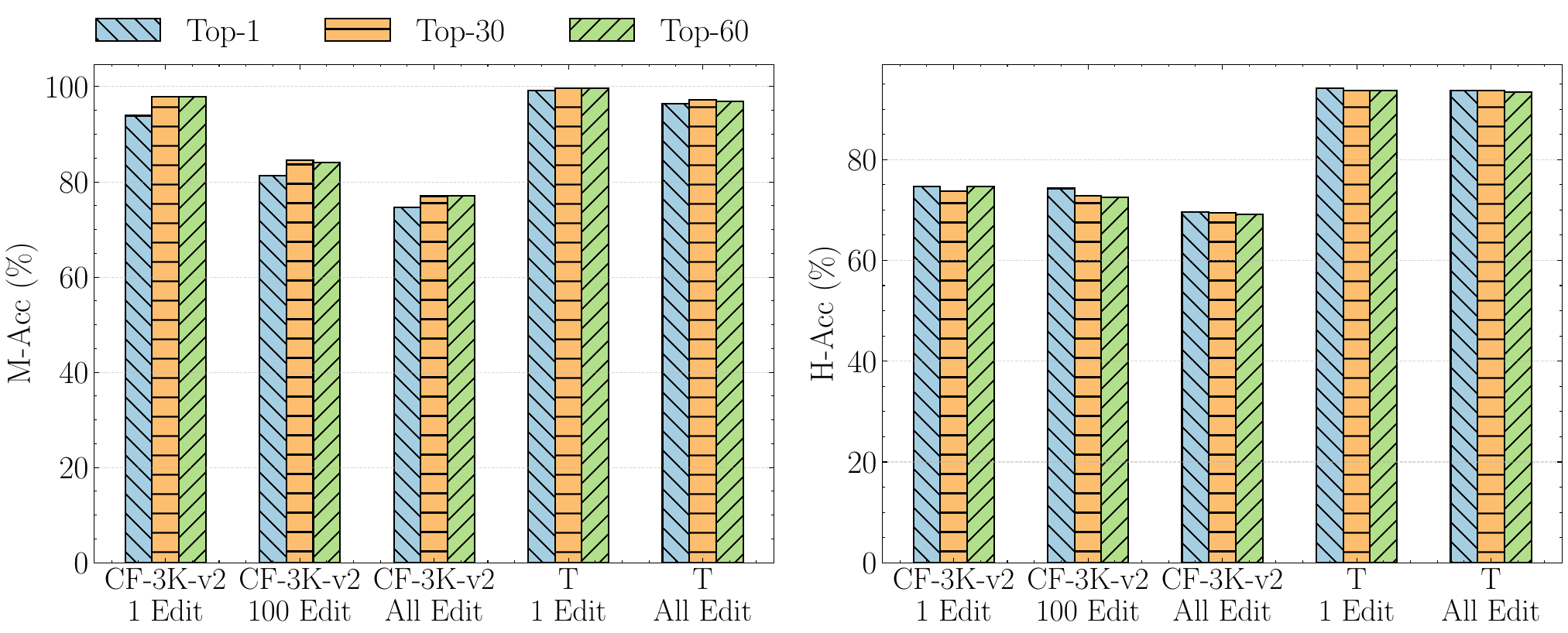}
    \caption{Top-$k$ of Maximum Inner Product Search (MIPS).}
    \label{fig:topk_macc_hacc}
\end{figure}

\subsection{Efficiency and Scalability Study}\label{sec:rq5}

\subsubsection{\textbf{Efficiency Study}}
We evaluate retrieval efficiency by comparing SimHash and FAISS against a brute-force cosine similarity baseline. As illustrated in Figure~\ref{fig:efficiency}, while all methods yield comparable accuracy across edit settings, the efficiency gap widens significantly as the number of edits scales. In the most demanding ``all-edited'' scenario, SimHash and FAISS reduce total query latency from 105.97s to 4.08s, a $25.9\times$ speedup, without sacrificing M-Acc or H-Acc. These results validate that our retrieval framework provides the high scalability necessary for real-time, large-scale knowledge editing applications. See Appendix~\ref{apc:rq2} for extended study on editing efficacy.

\subsubsection{\textbf{Impact of Retrieval Depth ($k$)}}
Figure~\ref{fig:topk_macc_hacc} examines how varying the retrieval depth $k \in \{1, 30, 60\}$ in Maximum Inner Product Search (MIPS) affects H-Acc. Counter-intuitively, larger $k$ values do not improve performance; $k=1$ achieves the highest accuracy across most scenarios. This suggests that a minimal set of highly relevant candidates is sufficient for multi-hop reasoning. Conversely, increasing $k$ inflates the token count and introduces extraneous or conflicting noise into the LLM's context, which degrades reasoning precision and increases inference overhead. Maintaining a concise context window ($k=1$) ensures a higher signal-to-noise ratio, which is critical for scalable QA systems operating under context length constraints.

\subsubsection{\textbf{Neural Scaling Law}}
To quantify how model capacity and edit volume influence performance, we analyze M-Acc and H-Acc scaling. As shown in Figure~\ref{fig:neural scaling law}, M-Acc follows the power-law, with a high $R^2$ confirming predictable scaling, providing a robust framework for estimating system behavior at scale.
The negative exponents characterize performance decay as complexity increases. Notably, sensitivity to dimensionality ($\beta = -0.48$) exceeds that of edit volume ($\alpha = -0.08$), suggesting excessive dimensionality triggers overfitting. 
In contrast, H-Acc a disparate scaling regime.
While the positive coefficients suggest marginal gains from increased capacity and data, the lower $R^2$ reveals a more stochastic behavior. H-Acc is governed by local graph properties and the stability of specific reasoning paths, making it highly sensitive to the Event Factorization Drift identified in our structural analysis.

\begin{figure}[t]
    \centering
    \includegraphics[width=0.99\linewidth, trim=0 10mm 0 0]{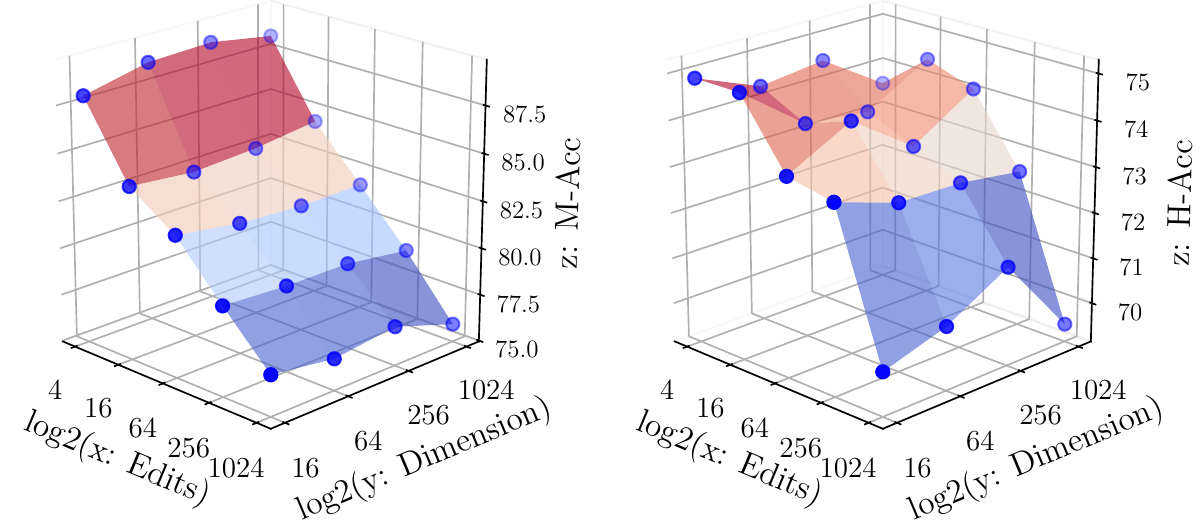}
    
    \caption{\\
        M-Acc Neural Scaling Law\hspace{0.8cm}H-Acc Neural Scaling Law\\
        $z = 93.57 \cdot x^{-0.08} + 1.62 \cdot y^{-0.48}$
        \hspace{0.35cm}
        $z = 74.44 \cdot x^{0.03} + 3.14 \cdot y^{0.23}$\\
        $R^2 = 0.9692$
        \hspace{2.9cm}
        $R^2 = 0.5673$ 
    }
    \vspace{-6mm}
    \label{fig:neural scaling law}

\end{figure}

\subsection{Ablation Study} \label{sec:rq2}
We conduct ablation studies to quantify the contribution of each component to \method{}'s performance, as summarized in \Cref{tab:ablation}.

\noindent\textbf{Topological Manifold.} 
Compared to a semantic-only baseline, omitting the topological manifold significantly degrades multi-hop QA performance. While semantic similarity identifies lexically relevant nodes, it fails to bridge lexically distant but structurally pivotal entities. The topological manifold captures higher-order relational neighborhoods, ensuring globally coherent reasoning chains, especially when lexical signals are sparse or ambiguous.

\noindent\textbf{Subject-Detection in Conflict Resolution.} 
Disabling span-based conflict resolution increases erroneous \textsc{Replace} operations, underscoring the necessity of subject-level cues for target disambiguation. Absent these cues, the system conflates semantically proximate but topologically distinct facts, triggering unintended overwrites. This module ensures that updates are grounded in the correct structural context, preserving the fidelity of the sequential knowledge stream.

\noindent\textbf{SimHash Indexing.} 
To evaluate the efficacy of our hyperedge linking, we replace the SimHash and FAISS-based candidate selection with random sampling. This results in a collapse of linking accuracy and a surge in incorrect edits. These findings confirm that our LSH-based retrieval is more than a latency optimization; it acts as a high-precision filter that prunes the search space to semantically and topologically congruent candidates.

\noindent\textbf{Topological LoRA Adaptation. } 
Removing the LoRA modules yields consistent accuracy degradation across all configurations, indicating that incremental HGNN adaptation is vital for maintaining representation quality as the knowledge base evolves. While the structural prior remains robust, LoRA facilitates parameter-efficient realignment with newly injected content.
By dynamically recalibrating entity-event proximities, LoRA effectively counters structural drift, mitigating the neighborhood reshaping effect, without inducing catastrophic forgetting in the backbone.

\noindent\textbf{Relational Atomicity: Hypergraph vs. Triple Structure. } 
A core thesis of our work is that $n$-ary hypergraphs prevent \textit{Event Factorization Drift}. Replacing the hypergraph structure with conventional binary triples leads to a sharp performance collapse across all evaluation settings (Table~\ref{tab:ablation}). 
The degradation confirms that binary reification fragments complex events, fracturing the semantic coupling required for multi-hop traversal. Hypergraphs preserve the atomic integrity of relational constraints, which is essential for consistent reasoning under structural drift.

\definecolor{qdlbg}{rgb}{0.94, 0.97, 1.0}
\par\noindent \textbf{Efficacy of Question Decomposition. } 
We examine the impact of decomposition quality. Initially utilizing a LLaMA-2-7B model from the KeDKG \cite{lu2025knowledge} framework, we observed limited generalization on complex updates. Replacing this with a GPT-4o-mini-based decomposer yielded significant gains (\Cref{tab:ablation}, rows in \colorbox{qdlbg}{blue}). This underscores that high-fidelity decomposition is a prerequisite for effective sequential editing, as it provides the necessary plan for the subsequent dual-manifold retrieval and reasoning synthesis steps.

\begin{table}[t]
\centering
\small
\setlength\tabcolsep{2pt}
\resizebox{0.99\columnwidth}{!}{%
\begin{tabular}{lcccc}
\toprule
\textbf{Feature} & \textbf{MeLLo} \cite{zhong2023mquake} & \textbf{PokeMQA} \cite{gu2024pokemqa} & \textbf{KeDKG} \cite{lu2025knowledge} & \textbf{\method{}} \\
\midrule
Representation  & Sentential  & Sentential   & Binary Graph & \textbf{$n$-ary Hyper.} \\
Fact Unit       & Sentence    & Sentence     & Triple       & \textbf{Hyperedge}      \\
Update Policy   & Lazy        & Lazy         & Per-Query    & \textbf{Batched}        \\
Decomposition   & LLM-driven  & LLM-driven   & Template     & \textbf{Template}       \\
Retrieval Unit  & Sentence    & Sentence     & Entity       & \textbf{Hyperedge}      \\
Retrieval Space & Semantic    & Sem.+Conf.   & Symbolic     & \textbf{Dual-Manif.}    \\
Retriever Ensemble        & \xmark      & \cmark       & \cmark       & \cmark                  \\
\bottomrule
\end{tabular}
}
\caption{Knowledge editing framework comparison.}
\label{tab:comparison}
\vspace{-8mm}
\end{table}

\section{Related Work}\label{sec:work}

\noindent \textbf{Sequential Knowledge Editing (SKE). }
SKE aims to modify factual priors in LLMs under non-stationary edit streams. \textit{Parametric} methods, such as ROME \cite{meng2022locating} and MEMIT \cite{meng2023mass}, perform localized weight interventions; however, sequential updates often trigger weight-space instability, leading to catastrophic forgetting \cite{hartvigsen2024aging} and multi-hop reasoning failures \cite{zhong2023mquake}. Meta-learning approaches like MEND \cite{mitchell2021fast} improve update scalability but remain sensitive to variance in continuous streams. Conversely, \textit{memory-based} editors like SERAC \cite{mitchell2022memory} and GRACE \cite{hartvigsen2024aging} bypass weight modification via external caches but typically treat edits as \textit{i.i.d.} key-value pairs, ignoring the latent structural dependencies between facts. Recent work has questioned whether the editing benchmarks and metrics properly reflect reliable editing behavior, especially in terms of locality and robustness under realistic evaluation settings \cite{liu2026evaluating,liu2025model}. As categorized in Table~\ref{tab:comparison}, \method{} diverges from these reactive paradigms through a \textit{Proactive Temporal Policy}: we introduce upfront structural adaptation to preserve the \textit{topological atomicity} required for consistent multi-hop reasoning across sequential edit streams.

\noindent \textbf{Retrieval-Augmented Generation (RAG).}
RAG decouples knowledge acquisition from parametric reasoning using external corpora \cite{lewis2020retrieval} or graphs \cite{edge2024global}. Retrieval reliability is inherently vulnerable to non-stationary distribution shifts \cite{thakur2021beir, gama2014survey} where query distributions diverge from training priors \cite{mallen2023not}. While retrieval editors \cite{han2023memory} typically assume static embedding manifolds, we identify \textbf{factorization drift}, the erosion of \nary{} event integrity within binary latent spaces, as a fundamental failure mode. \method{} introduces \textbf{Dual-Manifold Retrieval} to synchronize the semantic manifold with a structure-aware HGNN latent space, ensuring evidence remains topologically congruent across sequential updates.

\noindent \textbf{Symbolic Structure vs. Latent Synchronization.}
Property graphs and temporal knowledge graphs can preserve evolving facts through edge attributes, qualifiers, or temporal metadata at the symbolic level. However, such symbolic updates do not necessarily update the dense retriever's latent manifold, which may remain biased toward stale pre-edit priors. As a result, a knowledge base can be symbolically correct while still being functionally misaligned with neural retrieval, especially in multi-hop QA where constraints are often implicit, linguistically diverse, or distributed across multiple reasoning steps. In contrast, \method{} treats sequential editing as a neuro-symbolic synchronization problem: structural updates modify the \nary{} hypergraph state, while topological adaptation realigns the structure-aware embedding manifold with the evolved topology. This distinction moves beyond symbolic filtering or attribute matching, ensuring that updated events are both explicitly stored and retrievable through latent similarity.

\noindent \textbf{Hypergraph Representation Learning.}
Real-world knowledge comprises complex $n$-ary relations that resist lossless decomposition into binary triples \cite{wen2016learning}. Traditional KG embeddings \cite{bordes2013translating} typically rely on binary reification, which decouples semantic associations and induces \textit{reification bottlenecks}, i.e., the loss of joint coupling between participants. Hypergraph Neural Networks (HGNNs) \cite{feng2019hypergraph, yadati2019hypergcn, zheng2025modeling} offer a principled alternative by encoding $n$-ary relations as atomic topological units, effectively capturing higher-order correlations. While HGNNs have been utilized for static tasks like cross-document QA \cite{tu2020document}, their utility in sequential knowledge editing remains unaddressed. \method{} leverages the \textbf{topological atomicity} of hypergraphs to provide a robust substrate for dynamic updates. By treating events as unified hyperedges rather than fragmented triple-chains, we preclude the spurious composition failures characteristic of reified binary systems (see Table~\ref{tab:comparison}).

\section{Conclusion}
We present \method{}, a framework that mitigates \textit{structural drift} in sequential knowledge editing. By replacing fragmented binary triples with atomic hyperedges and leveraging $O(1)$ SimHash-based alignment, \method{} ensures LLM reasoning remains grounded in evolving topologies. Experimental results demonstrate that \method{} achieves state-of-the-art H-Acc (96.24\% and 21.06\%) on complex benchmarks, mitigating the severe accuracy collapses (42.9\%--88.3\%) inherent in KG-based variants. Combined with a $25.9\times$ retrieval speedup, these findings underscore that topological integrity is as critical as parametric fidelity for the development of reliable, long-lived knowledge-augmented LLMs.

\begin{acks}
This work is partially supported by the National Science and Technology Council (NSTC), Taiwan (Grants: NSTC-114-2640-E-A49-011, 114-2222-E-A49-004, and 114-2639-E-A49-001-ASP).
\end{acks}

\newpage
\bibliographystyle{ACM-Reference-Format}
\bibliography{kdd}

\appendix

\section{Extended Experiments}\label{ap:extend}

\subsection{Editing Efficacy} \label{apc:rq2}

\noindent \textbf{Hyperedge Linking Efficacy. }
We specifically analyze the \textsc{Replace} editing operation, as it represents the most critical bottleneck for structural integrity. In this setting, the system must accurately localize an existing hyperedge to be superseded; a failure in linking directly results in a missed edit or an erroneous graph modification. In contrast, \textsc{Add} operations are non-destructive and less sensitive to linking precision, as they treat incoming knowledge as independent structural expansions.
Figure~\ref{fig:Hyperedge linking} (left) illustrates the F1-score performance of \method{} compared to brute-force cosine similarity and random selection across various top-$k$ retrieval thresholds. While the cosine baseline represents an exhaustive upper bound for retrieval accuracy, it is computationally prohibitive in large-scale sequential environments. \method{} achieves $F_1$ scores nearly identical to cosine similarity while operating on a significantly pruned search space. By utilizing SimHash signatures and binary indexing, \method{} balances high-dimensional retrieval precision with orders-of-magnitude improvements in computational latency (detailed in \Cref{sec:rq5}), offering a scalable alternative to traditional dense similarity metrics.

\begin{figure}[t]
    \centering
    \includegraphics[width=0.9\linewidth]{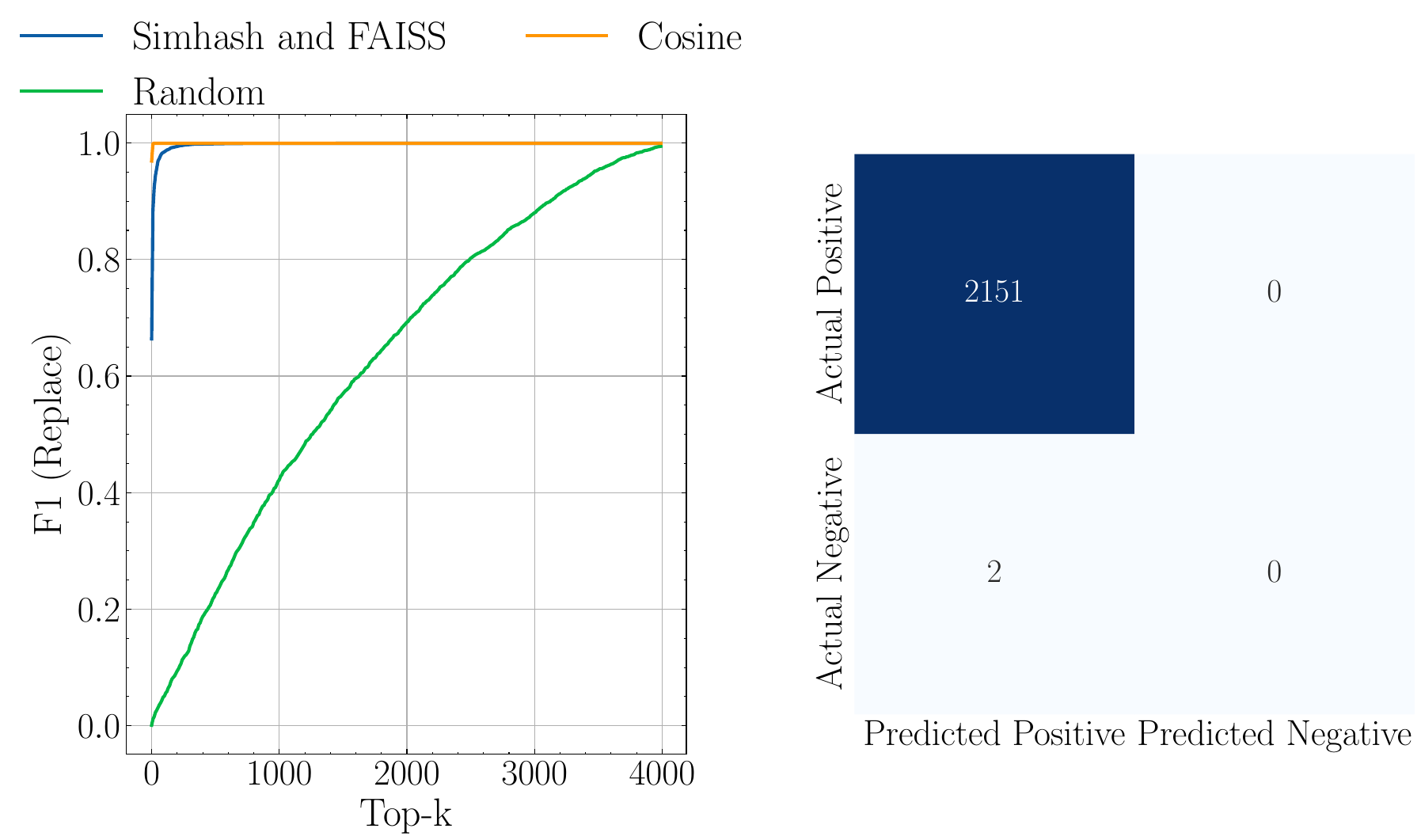}
    \vspace{-4mm}
    \caption{Hyperedge linking efficacy and conflict resolution accuracy. (Left) $F_1@k$ performance for identifying \textsc{Replace} operations across varying retrieval depths. (Right) Confusion matrix validating the discriminative accuracy of \method{}'s topological conflict resolution logic.}
    \label{fig:Hyperedge linking}
    \vspace{-6mm}
\end{figure}

\begin{table}[t]
\centering
\begin{tabular}{lccccc}
\toprule
\textbf{Method} & \textbf{Total} & \textbf{FAISS} & \textbf{Embed.} & \textbf{Graph} & \textbf{CF-All} \\
\midrule
KeDKG & 5.68 & -- & -- & 5.68 & 36.77 \\
\method{} & 64.21 & 0.0665 & 57.75 & 6.40 & \textbf{69.57} \\
\bottomrule
\end{tabular}
\caption{Memory footprint comparison at 3,000 edits. Memory values are reported in MB. CF-All denotes Hop-wise Accuracy on the MQuAKE-CF-3K-v2 all-edit setting.}
\vspace{-10mm}
\label{tab:memory_footprint}
\end{table}

\noindent \textbf{Topological Conflict Resolution Accuracy.}
To evaluate the discriminative reliability of the decision logic $\Psi$, we assess its capacity to disambiguate \textsc{Replace} and \textsc{Add} operations. Figure~\ref{fig:Hyperedge linking} (right) presents the confusion matrix for the \textsf{MQuAKE-CF-3k-v2} dataset, utilizing ground-truth labels based on semantic intent. The results reveal exceptional discriminative fidelity: \method{} correctly linked 2,151 of 2,153 \textsc{Replace} cases, achieving a 99.9\% precision rate with zero missed matches. This high precision is critical for hypergraph stability, ensuring that topological modifications are anchored to existing structures without inducing unintended overwrites. The ability to maintain such accuracy despite the linguistic variance of $n$-ary facts further validates the robustness of our span-based alignment mechanism.

\begin{figure}[t]
\centering
\footnotesize

\includegraphics[width=0.8\linewidth]{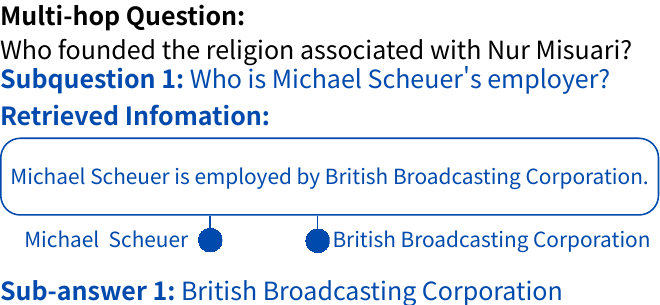}\\[-1mm]
\textbf{(a)}\\[-0.5mm]

\includegraphics[width=0.8\linewidth]{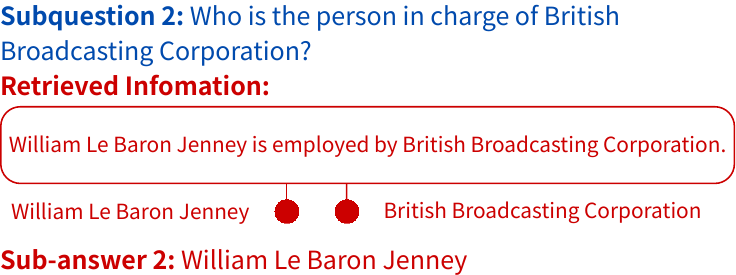}\\[-1mm]
\textbf{(b)}\\[-0.5mm]

\includegraphics[width=0.8\linewidth]{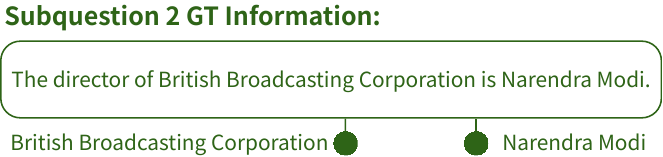}\\[-1mm]
\textbf{(c)}

\vspace{-1mm}
\caption{Retrieval-induced recall error. (a) successful evidence grounding in the first reasoning hop; (b) a relational mismatch in the second hop, where entity-centric overlap dominates the retrieval signal; (c) the target hyperedge present in the updated topology but missed by the retriever.}
\Description{Three panels illustrating a retrieval-induced recall error: successful first-hop grounding, erroneous retrieval via entity overlap, and the missed ground-truth relational hyperedge.}
\label{fig:recall_error}

\vspace{-3mm}
\end{figure}

\begin{figure}[t]
\centering
\footnotesize

\includegraphics[width=1\linewidth]{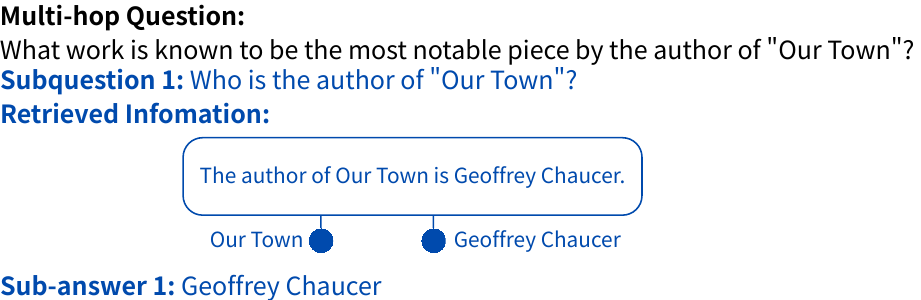}\\[-1mm]
\textbf{(a)}\\[-0.5mm]

\includegraphics[width=1\linewidth]{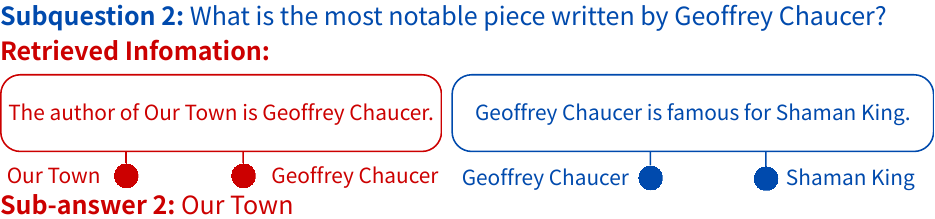}\\[-1mm]
\textbf{(b)}\\[-0.5mm]

\includegraphics[width=0.5\linewidth]{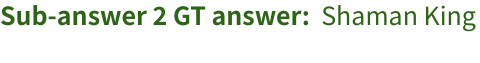}\\[-1mm]
\textbf{(c)}

\vspace{-1mm}
\caption{Reasoning-induced contextual noise. (a) valid initial reasoning step; (b) introduction of a redundant distractor that induces contextual interference; (c) the correct ground-truth hyperedge, which was successfully retrieved but ignored by the reasoning module.}
\Description{Three panels illustrating reasoning-induced contextual noise: a valid initial reasoning step, redundant or conflicting information, and the correct ground-truth hyperedge.}
\label{fig:noising_error}
\vspace{-6mm}
\end{figure}

\begin{table*}[h]
\centering
\small
\resizebox{1\textwidth}{!}{%
\begin{tabular}{p{4cm} p{4.5cm} p{9.6cm}}
\toprule
\multicolumn{3}{c}{
    \begin{tabular}{@{} p{7cm} | p{8.6cm} @{}}
        \textbf{Old Knowledge:} & \textbf{New Knowledge:} \\
        \textsf{Ciudad de Murcia is associated with the sport of association football.} & \textsf{Ciudad de Murcia is associated with the sport of basketball.} \\ [-20pt]
        & \\ \cmidrule{2-2} 
         & \textbf{Multi-hop Question:} \\
        \textsf{association football was created in the country of England.} & \textsf{What is the capital of the country where the sport associated with} \\
        \textsf{The capital of England is London.} & \textsf{Ciudad de Murcia was originated?}\\
        \textsf{basketball was created in the country of Soviet Union.} & \\
        \textsf{The capital of Soviet Union is Russellville.} & \\
    \end{tabular}
} \\ 
\toprule
\textbf{Stage} & \textbf{Knowledge Graph (KEDKG)} & \textbf{Hypergraph (HyperPatch)} \\
\midrule
Initial Graph State & \textsf{(1). Ciudad de Murcia $\rightarrow$ sport $\rightarrow$ association football $\rightarrow$ country $\rightarrow$ England $\rightarrow$ capital $\rightarrow$ London (2). basketball $\rightarrow$ country $\rightarrow$ Soviet Union $\rightarrow$ capital $\rightarrow$ Russellville} &  \textsf{(1). Ciudad de Murcia $\leftarrow$ [Ciudad de Murcia is associated with the sport of association football.] $\rightarrow$ association football $\leftarrow$ [association football was created in the country of England. ] $\rightarrow$ England $\leftarrow$ [The capital of England is London. ] $\rightarrow$ London (2). basketball $\leftarrow$ [basketball was created in the country of Soviet Union.] $\rightarrow$ Soviet Union $\leftarrow$ [The capital of Soviet Union is Russellville.] $\rightarrow$ Russellville} \\
\midrule
New Knowledge Extraction & \textsf{(Ciudad de Murcia, is associated, basketball)} & \textsf{Ciudad de Murcia is associated with the sport of basketball. | Ciudad de Murcia | basketball}\\
\midrule
Post-Edit Graph State & \textsf{Ciudad de Murcia $\rightarrow$ (sport $\rightarrow$ association football $\rightarrow$ country $\rightarrow$ England $\rightarrow$ capital $\rightarrow$ London) \& (is associated $\rightarrow$ basketball $\rightarrow$ country $\rightarrow$ Soviet Union $\rightarrow$ capital $\rightarrow$ Russellville)} &  \textsf{(1). association football $\leftarrow$ [association football was created in the country of England. ] $\rightarrow$ England $\leftarrow$ [The capital of England is London. ] $\rightarrow$ London (2). Ciudad de Murcia $\leftarrow$ [Ciudad de Murcia is associated with the sport of basketball.] $\rightarrow$ basketball $\leftarrow$ [basketball was created in the country of Soviet Union.] $\rightarrow$ Soviet Union $\leftarrow$ [The capital of Soviet Union is Russellville.] $\rightarrow$ Russellville} \\
\midrule
Multi-hop Reasoning Path & \textcolor{red}{\textsf{$a^{(1)}$: association football}}, \textcolor{red}{\textsf{$a^{(2)}$: England}}, \textcolor{red}{\textsf{$a^{(3)}$: London}}  & \textcolor{blue}{\textsf{$a^{(1)}$: basketball}}, \textcolor{blue}{\textsf{$a^{(2)}$:  Soviet Union}}, \textcolor{blue}{\textsf{$a^{(3)}$: Russellville}} \\ 
\bottomrule
\end{tabular}
}
\caption{Qualitative comparison of knowledge editing between standard KG and \nary{} Hypergraph (3-hop). \textcolor{blue}{Blue} indicates correct answers; \textcolor{red}{red} indicates incorrect answers from spurious composition.}
\label{tab:qualitative_example_3_hop}
\vspace{-8mm}
\end{table*}

\subsection{Memory Footprint Analysis}
Memory serves as a vital substrate for preserving relational atomicity under \nary{} structural drift. Although \method{} incurs higher storage overhead than triple-based methods, the additional cost remains modest in absolute terms. At 3,000 edits, \method{} requires 64.21MB, compared with 5.68MB for KeDKG. This overhead yields a substantial gain in reasoning reliability, improving CF-All H-Acc from 36.77\% to 69.57\%.
As shown in Table~\ref{tab:memory_footprint}, the graph structure of \method{} occupies only 6.40MB, which is comparable to KeDKG. Most of the additional memory comes from the structure-aware manifold embeddings, which require 57.75MB. This indicates that the memory increase is not caused by excessive graph storage, but by the latent representations needed to resolve \nary{} drift in non-stationary environments. In contrast, triple-based reification induces structural fragmentation, leading to severe accuracy degradation despite its smaller memory footprint.

\subsection{Error Analysis}
To provide a granular understanding of system limitations, we categorize failure modes into two primary types: \textbf{retrieval-induced recall errors} and \textbf{reasoning-induced contextual noise}.
Crucially, we observe that these errors stem from the underlying hypergraph-based RAG pipeline rather than the knowledge editing mechanism itself. In the following cases, while the hypergraph topology was successfully updated to reflect the new facts, the downstream retrieval and reasoning modules, which remain agnostic to our editing proposal, introduced independent architectural bottlenecks.

\noindent \textbf{\textbf{Recall Error}}
Figure~\ref{fig:recall_error} illustrates a recall failure originating from the retrieval manifold rather than a failed edit. In the first reasoning hop (Figure~\ref{fig:recall_error}(a)), the system correctly identifies Michael Scheuer's employer as the British Broadcasting Corporation (BBC). However, in the subsequent hop (Figure~\ref{fig:recall_error}(b)), the retriever erroneously fetches a hyperedge related to William Le Baron Jenney. While the edit successfully placed the target fact (Narendra Modi as director, Figure~\ref{fig:recall_error}(c)) into the graph, the retrieval module favored entity overlap (BBC) over the specific relational intent of the query. This shows that recall errors belong to the base retriever's limitations and do not indicate a failure in the structural update logic.

\noindent \textbf{Contextual Noise}
Figure~\ref{fig:noising_error} demonstrates a reasoning-induced error. As shown in Figure~\ref{fig:noising_error}(a), the initial sub-question is resolved correctly. During the second hop, the system successfully retrieves the edited ground-truth fact (Geoffrey Chaucer's association with ``Shaman King'', Figure~\ref{fig:noising_error}(c)). However, the RAG context is contaminated by a redundant, noisy hyperedge from the previous step (Figure~\ref{fig:noising_error}(b)). This distraction causes the generator to select the incorrect sub-answer. This failure is purely a reasoning bottleneck within the LLM-based RAG process; the fact that the correct edited knowledge was retrieved confirms that the knowledge editing component performed as intended.

\subsection{Qualitative Study}\label{ap:casestudy}
As shown in Table~\ref{tab:qualitative_example_3_hop}, the initial knowledge states that \textsf{Ciudad de Murcia is associated with the sport of association football}, and that \textsf{association football was created in the country of England}, whose capital is London. A new fact later replaces the sport associated with \textsf{Ciudad de Murcia} to \textsf{basketball}, prompting an edit. In the knowledge graph setting, each triple is decomposed into binary relations such as \textsf{(Ciudad de Murcia $\rightarrow$ sport $\rightarrow$ association football)} and \textsf{(association football $\rightarrow$ country $\rightarrow$ England)}. After the new fact is introduced, the edited graph appends the new relation \textsf{(Ciudad de Murcia $\rightarrow$ sport $\rightarrow$ basketball)} while retaining the old one, causing both to coexist. As a result, the downstream multi-hop path \textsf{Ciudad de Murcia $\rightarrow$ association football $\rightarrow$ England $\rightarrow$ London} remains reachable, yielding outdated answers. In contrast, the hypergraph structure preserves each sentence as an atomic hyperedge, maintaining semantic integrity. The original hyperedge indicating that Ciudad de Murcia is associated with association football is removed entirely and replaced by a new hyperedge stating that it is associated with basketball. As a result, the updated multi-hop path proceeds through \textsf{basketball}, then to \textsf{Soviet Union}, and finally reaches the correct answer \textsf{Russellville}.

\end{document}